\documentclass[sigconf]{acmart}
\usepackage{booktabs} 

\usepackage{mystyle}
\usepackage{subfigure}
\usepackage{algorithm}
\usepackage[noend]{algorithmic}
\usepackage{boxedminipage}
\usepackage{mathtools}
\usepackage{wrapfig}
\usepackage{url}

\renewcommand{\Paragraph}[1]{\paragraph*{\hspace{-.15in} \sffamily \textbf{#1:}}}

\title{Coresets for Kernel Regression}

\begin{document}
 
\author{Yan Zheng}
\affiliation{%
  \institution{University of Utah}
  \city{Salt Lake City} 
  \state{UT} 
  \postcode{84108}
}
\email{yanzheng@cs.utah.edu}

\author{Jeff M. Phillips}
\authornote{Thanks to NSF CCF-1350888, IIS-1251019, ACI-1443046, CNS-1514520, CNS-1564287.}
\affiliation{%
  \institution{University of Utah}
  \city{Salt Lake City} 
  \state{UT} 
  \postcode{84108}
}
\email{jeffp@cs.utah.edu}
\date{}

\begin{abstract}  
Kernel regression is an essential and ubiquitous tool for non-parametric data analysis, particularly popular among time series and spatial data.  However, the central operation which is performed many times, evaluating a kernel on the data set, takes linear time.  This is impractical for modern large data sets.  

In this paper we describe coresets for kernel regression: compressed data sets which can be used as proxy for the original data and have provably bounded worst case error.  The size of the coresets are independent of the raw number of data points; rather they only depend on the error guarantee, and in some cases the size of domain and amount of smoothing.  We evaluate our methods on very large time series and spatial data, and demonstrate that they incur negligible error, can be constructed extremely efficiently, and allow for great computational gains.  
\end{abstract}

\maketitle

\section{Introduction}
\label{sec:intro}

Kernel regression~\cite{Nad64,Wat64} is a powerful non-parametric technique for understanding scalar-valued 1-dimensional (and higher dimensional) data sets.  It has distinct advantages over linear or polynomial regression techniques in that it does not impose a possibly restrictive or over-fitting model on the data.  Rather it uses a kernel similarity function to describe a smooth weighted average over the points.  This allows the predicted function to locally adapt to the values of the data.  
These advantages have led to wide use of the kernel regression to predict, model, and visualize data from 
stocks~\cite{Wol00} 
to weather monitoring~\cite{SSIS11} 
to quantified self~\cite{Tom16}.  

However, as these data sets have grown to enormous scale, these kernel regression techniques have hit computational bottlenecks.  Just to evaluate the model at a single query point takes $O(|P|)$ time.  This runtime is completely infeasible in comparison to parametric models where it takes $O(1)$ time, especially in a common situation where many (perhaps $O(|P|)$ queries) are made.  
This paper explores near-linear and sub-linear techniques to compress data for kernel regression, to bring these powerful approaches to large data.  

\subsection{Basic Definitions}
\label{sec:defs}

Kernel regression is based on a \emph{kernel}, which is a bivariate function $K : \b{R}^d \times \b{R}^d \to [0,1]$ which in this setting we can restrict to map to a range of $[0,1]$ without loss of generality.  As input points get closer, the kernel value should get larger.  In this paper we will focus on Gaussian kernels so $K(p,q) = \exp(-\|p-q\|^2/\sigma^2)$ for a smoothing \emph{bandwidth} parameter $\sigma$.  These kernels have nice properties since they are smooth, but other kernels such as Laplace, Epanechnikov, or Triangle which have Lipschitz bounds should be perfectly suited for any of our results.  

We consider an input data sets $P \subset \b{R}^{d+1}$.  We decompose this into the first $d$ explanatory coordinates denoted $P_x \subset \b{R}^d$ and the last dependent one $P_y \subset \b{R}$.  
Most examples we discuss have $d=1$ where it is common to think of these data items $P_x$ as times, but many approaches generalize for larger values of $d$.  
Then each data item $p \in P$ is also associated with a scalar data value $p_y$ (the set of these comprises $P_y$).  

A \emph{kernel density estimate} (KDE) is a smooth function defined by convolving the data set with the kernel, for a query point $q \in \b{R}^d$ as 
\[
\kde_P(q) = \frac{1}{|P|} \sum_{p \in P} K(p_x, q).
\]
We say a \emph{weighted kernel density estimate} (WKDE) replaces this with a weighted sum
\[
\wkde_P(q) = \frac{1}{|P|} \sum_{p \in P} K(p_x, q) p_y.
\]
Finally, the (Nadaraya-Watson) \emph{kernel regression} function is defined for a query point $q \in \b{R}^d$ as
\[
\reg_{P}(q) = \frac{\sum_{p \in P} K(p_x,q) p_y}{\sum_{p \in P} K(p_x,q)} = \frac{\wkde_P(q)}{\kde_P(q)}. 
\]

 \begin{figure}[t!]
  \includegraphics[width=0.49\linewidth]{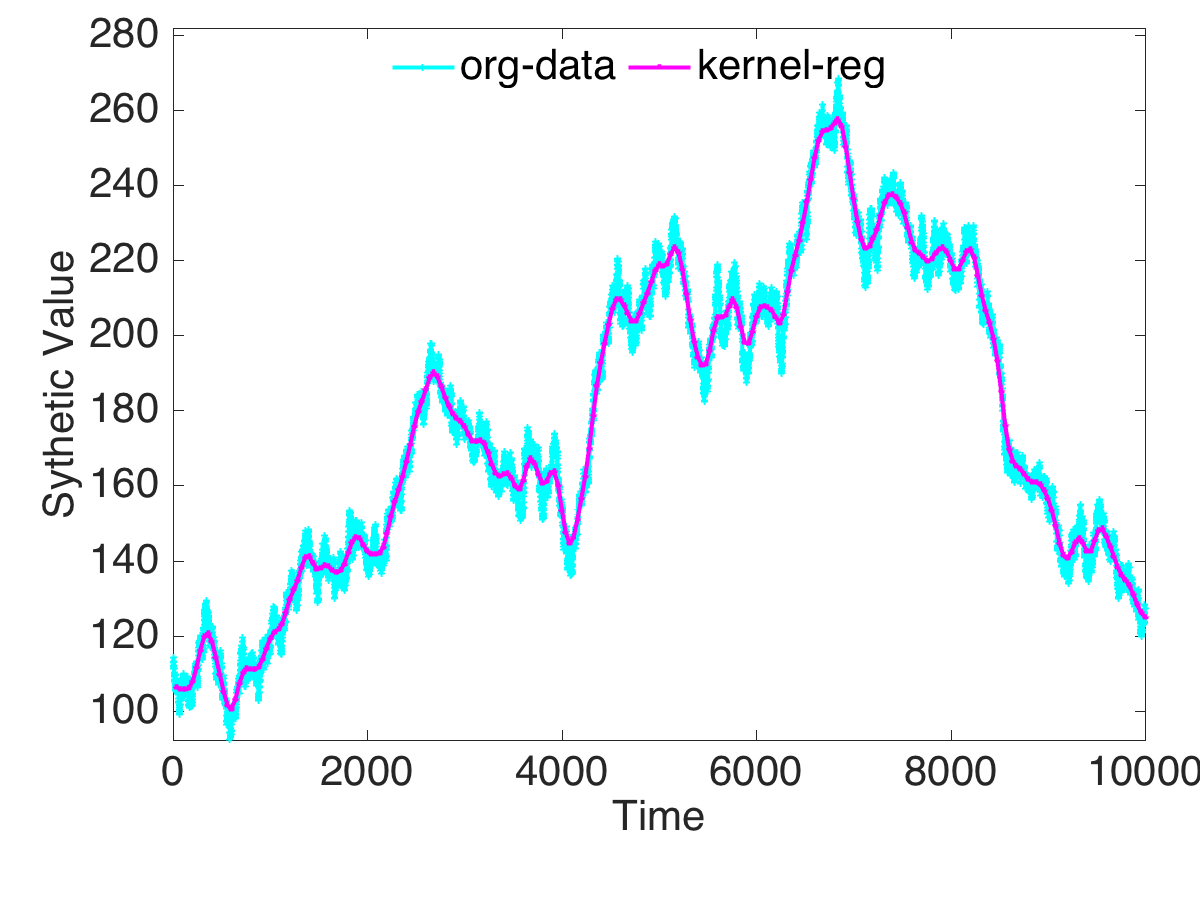}
  \includegraphics[width=0.49\linewidth]{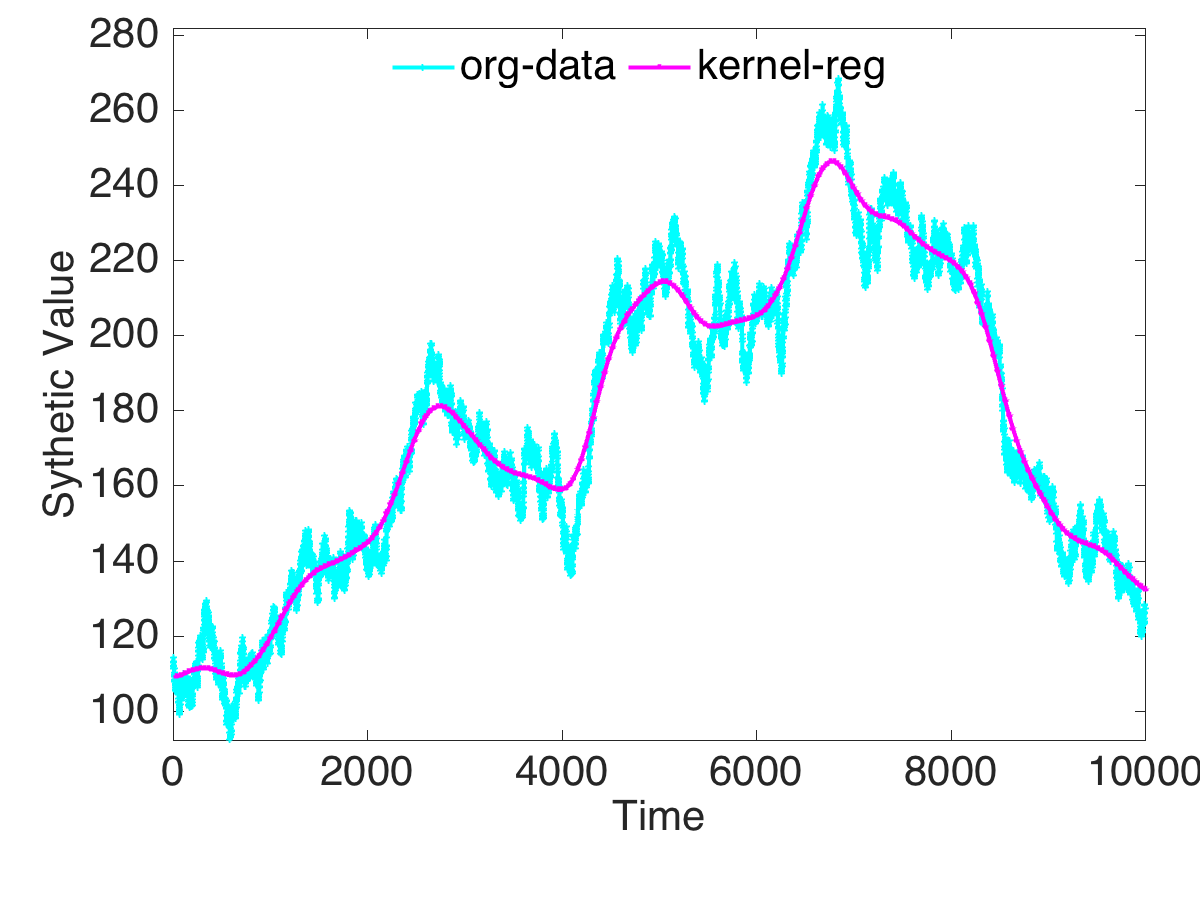}
  \caption{Kernel regression of synthetic data with bandwidth 50 (left) and 200 (right).}
   \label{fig:syn_vis}
\end{figure}
  
This maps each domain point in $\b{R}^d$ to an estimate in the space $\b{R}$ of scalar values; it takes a weighted average (defined by kernel similarity) of the scalar values nearby.  Figure \ref{fig:syn_vis} visualizes, for $d=1$, the kernel regression and its original data of a synthetic time series dataset (see Section \ref{sec:exp}) with bandwidth $50$ and $200$.  The query $q$ can be any time point between $0$ and $10{,}000$.  

There are various other forms of kernel regression~\cite{RW06}.  
But in this paper we focus on the Nadaraya-Watson variety~\cite{Nad64,Wat64} as it has an important long history and has been widely used in areas such as image processing \cite{takeda2007kernel} and economics \cite{blundell1998kernel}.   Moreover, since it does not try to pass through every data point, it is the most robust to outliers that are pervasive in large data, which often by necessity cannot be carefully filtered.  
	
\subsection{Coresets}

The brute force solution of kernel regression is time consuming as each computation  calculates $\kde$ and $\wkde$, which takes $O(|P|)$ time. In this paper, we show how to scalably apply kernel regression to massive scalar-valued data sets.  The main idea is to approximate $P$ with a \emph{coreset} $S$ where $S_x$ can for instance be a subset of $P_x$, but each $s \in S$ can potentially be given a scalar value $s_y$ different from the associated original point. In particular, the coreset $S$ should act as a proxy for $P$, so that for any $q$ the value $\reg_S(q)$ should approximate $\reg_P(q)$.  

The coreset $S$ should be substantially smaller than $P$, while also preserving the strong approximation guarantees.  Any query to $\reg_S$ takes time at most proportional to $|S|$ instead of $|P|$, so the size of $S$ directly impacts the efficiency of interacting with $\reg_S$.  
Moreover, if the construction of $S$ is efficient (and ours is roughly as fast as reading the data, or sorting if needed), then the time to compute $m$ values of the kernel regression (common for say visualization) is also reduced by $|P|/|S|$, after factoring the build time.  Here are a list of scenarios where such coresets are essential.
\begin{itemize}
\item 
The data is too big to store. For example, Square Kilometer Array, the world's largest radio telescope, receives several terabytes of data per second.  Most of the data is in scalar values, such as baseline-corrected power flux density, sensitivity, and receiver temperature, so kernel regression is a good way to track those scalar values over time.  
However storing all of this data is a challenging problem, let alone analyzing it.  Instead of storing all of it, a coreset for kernel regression would keep relevant data that provably behaves like the original data, but needs much less space.
\item 
The data is large and the older parts requires less accuracy.  For instance in analyzing trends in system log data, we want more accurate recent data, but allow more imprecision in historical data.  
For example, in the CloudLab~\cite{ricci2014introducing} central power database, power data serves as a way to monitor the cloud performance. It has scalar values and changes gradually overtime but may have noisy fluctuations. Kernel regression is a good way to track this, and older data can be kept with less precision.  
\item 
The data is for interactive analysis.  To interact with very large data stored on disk one can first analyze a small coreset, and then refine to a larger coresets with more accuracy as more precision is needed; this is much more efficient than bringing all relevant data to disk for each query.  Instead we can maintain several layers of different sized coresets.  For instance, in spatial data systems, such as Mesowest~\cite{horel2002mesowest}, temperature is connected with each geo-coordinate, to show temperature across the United States, a coarse level coreset is sufficient.  But to zoom the map to see the temperature at the state or city level, then a more detailed coreset is required.  
\end{itemize}

\subsection{Our Approach}
\label{sec:approach}

To formalize the meaning of $\reg_S(q)$ is ``close'' to $\reg_P(q)$, we focus on worst-case error guarantees ($L_\infty$ as opposed to $L_2$ or $L_1$ more common to KDEs); this ensures we do not have any spurious regression values.  This is essential for data analysis, since we want to be able to find important trends and detect outlier points, and also not be fooled into thinking we observe a non-existent trend or a non-existent outlier.  

Beyond that, the error function should not be affected by either a shift or a scaling of a scalar value, since this is equivalent to changing the units (e.g., Celsius to Fahrenheit).   
As such a natural bound will be absolute error difference with the bound depending on some quantity that depends on the scaling.  
We will use $M = \max_{p,p' \in P} |p_y - p'_y|$, the maximum difference between scalar values, so as the scale of the units on $p_y$ changes, $M$ does at the same rate.   

In particular, we are interested in a coreset $S$ of a data set $P$ such that for some domain $\c{U} \subset \b{R}^d$ that
\[
\max_{q \in \c{U}} | \reg_P(q) - \reg_S(q) | \leq \eps M.
\]
Our coresets $S$ have size depending linearly on $1/\eps$ and sometimes $\Delta = \max_{p, p' \in P} \|p_x - p'_x\|/\sigma$.  

It is worth noting that in setting $\c{U} = \b{R}^d$, such a result may not be possible.  The kernel regression definition $\reg_P(q)$ has in its denominator $\kde_P(q)$, so when $\kde_P(q)$ is very close to $0$, then $\reg_P(q)$ is very unstable.  So we consider a domain $\c{U}$ which is defined by a mild condition on $\kde_P(q)$; in particular that $\kde_P(q)$ is above some very small value $\rho$.  

To further put this error bound in perspective, consider instead the relative error	
$
\max_{q \in \c{U}} \frac{\reg_S(q)}{\reg_P(q)}.
$
This is unstable whenever $\reg_P(q)$ is close to $0$.  And, furthermore, the $q$ where $\reg_P(q)$ is small, depends entirely on the units chosen for the $p_y$ values.  For instance, we could have $p_y = 32^\circ$ Fahrenheit (not be close to $0$) or $p_y = 0^\circ$ Celsius which is exactly $0$ and makes \emph{any} relative error requirement imply no error at all.  Since the change of units is meaningless, this error measure is not feasible.  
	
\subsection{Our Results}
\label{sec:results}

Our results focus mainly on $P_x \subset \b{R}^1$ and $P_x \subset \b{R}^2$ (so the $x$-coordinate(s) naturally represent time or spatial coordinates), although many aspects extend naturally to high dimensions.  

We first bound the accuracy of a kernel regression coreset formed by random sampling;  these are the first known bounds for the \emph{sample complexity} of kernel regression.  It is of particular interest since in many cases the data set provided on input is itself a random sample from some much larger data set or distribution we do not have access to (e.g., a $1\%$ stream from Twitter).  So if the input data is indeed sampled, our bounds measure the error present before any analysis is applied.  
However, random sampling performs poor compared to most other methods we consider, so it makes sense to further compress them.  

We analyze (theoretically and empirically) several straight-forward aggregation techniques to construct coresets.  These are of particular interest since they mimic common online aggregation techniques~\cite{Bru00}.  We also propose some modifications which demonstrate sizable empirically improvements.  
Interestingly, effective coresets for KDEs~\cite{big-kde}, do not perform the best for kernel regression.  
 
\begin{algorithm}
  \caption{\textbf{Z-order (Z)}}
  \label{alg:algZorder}
    \begin{algorithmic}[1]
    \STATE Sort data $P_x$ in Z-order;  set $h = |P|/|S|$
    \STATE Choose a random number in $r = [0, h-1]$
    \FOR {$i \leftarrow 1$ to $|S|$}
    	\STATE Put $P_{r+h\cdot (i-1)}$ into $S$
    \ENDFOR
    \STATE return $S$
   \end{algorithmic}
 \vskip1pt
\end{algorithm}	
\begin{algorithm}
  \caption{\textbf{Z-Aggregate (ZA)}}
  \label{alg:algAgg}
    \begin{algorithmic}[1]
    \STATE Sort data $P_x$ in Z-order;  set $h = |P|/|S|$
    \FOR {$i \leftarrow 1$ to $|S|$}
    	\STATE $P_i = [P_{h \cdot (i-1)}, \cdots, P_{h\cdot i }] $
    	\STATE  Put average of all the points in $P_i$ into $S$
    \ENDFOR
    \STATE return $S$
   \end{algorithmic}
\end{algorithm}	

In particular, our recommended method \textbf{G-Aggregate} for $P_x \subset \b{R}^1$, carefully aggregates data over a fixed size non-empty grid cells; it takes $O(|P|)$ after sorting the data.  
For $P_x \subset \b{R}^2$, we recommend \textbf{Aggregate-Neighbor}, which carefully adds a few points to the coreset from \textbf{G-Aggregate}.  For a data sets $P_x \subset \b{R}^{d}$, these both produce a coreset $S$ of size $O(\Delta /\eps \rho)^d$, where $\Delta = \max_{p,p' \in P} \|p_x - p'_x\|/\sigma$, and guarantees for any $q \in \b{R}^d$ with $\kde_{P_x}(q) > \rho$ that 
$
| \reg_P(q) - \reg_S(q) | \leq \eps M,
$
where $M = \max_{p,p' \in P} |p_y - p'_y|$.
Moreover, these methods are simple to implement and work extremely well on real and synthetic data sets.

\subsection{Related Work}
This is the first work to address sample complexity and coreset size for Nadaraya-Watson kernel regression.  There is an enormous body of work on other types of coresets, see the recent survey on coresets~\cite{Phi16}, including many for parametric regression variants like least-square regression~\cite{boutsidis2012near} and $l_p$ regression~\cite{dasgupta2009sampling}.  

The only non-parametric regression coreset we are aware of is a form of kernel regression~\cite{wei2008theoretical} related to the smallest enclosing ball.  It predicts the value at a point $q \in \b{R}^d$ as $f(q) = \beta + \sum_{p \in P} \alpha_p K(p_x,q)$ with loss function $\sum_{p \in P} \max\{ 0, | f(p_x) - p_y| - \bar \eps \}$, for a parameter $\bar{\eps}$.  Then it finds a set of $O(1/\eps)$ non-zero $\alpha_p$ parameters (corresponding with points in the coreset) so many points satisfy $|f(p_x) - p_y| \leq \bar \eps (1+\eps)$.  No implementations were attempted.  

Rather, we believe the most related work involves coresets for kernel density estimates~\cite{JoshiKommarajuPhillips2011,BFLRSW13,Phillips2013,big-kde} as mentioned above.  We extend some of these results and show others do not work well when translated to the regression variant of this problem.  

\section{Subset Selection Methods}
\label{sec:methods}
We next describe several natural approaches to compress scalar-valued spatial data.  Some of these are likely in use in existing data aggregation frameworks (e.g., RFF~\cite{Bru00}), but as far as we know have not been analyzed in how they preserve kernel regression values.  

\Paragraph{Random sampling (RS)}
This method simply draws a uniform random sample $S$ from the data set $P$.  This is probably the most common data aggregation method anywhere.  In other cases, it is often assumed that even before aggregating data, the data is only a random sample of some unseen larger ``true'' dataset.  
This is known to approximate kernel density estimates~\cite{JoshiKommarajuPhillips2011,GBRSS12,BFLRSW13}, and we will show extends to kernel regression.  

\begin{algorithm}
  \caption{\textbf{G-Aggregate (GA)}}
  \label{alg:galgAgg}
    \begin{algorithmic}[1]
    \STATE Map $P_x$ into grid $G_\gamma$ 
    \FOR {$g \in G_\gamma(P)$} 
    	\STATE  Put average of all the points in $P_g$ into $S$
    \ENDFOR
    \STATE return $S$
   \end{algorithmic}
 \vskip1pt
\end{algorithm}	
\begin{algorithm}
  \caption{\textbf{Aggregate-Neighbor (AN)}}
  \label{alg:algAN}
    \begin{algorithmic}[1]
    \STATE Map $P_x$ into grid $G_\gamma$ 
    \FOR {$g \in G_\gamma(P)$} 
    	\STATE  Put average of all the points in $P_g$ into $S$
    \ENDFOR
    \FOR {$g \in G_\gamma(\bar P)$ adjacent to $G_\gamma(P)$} 
    \STATE For center $c$ of $g$, put $(c, \reg_P(c))$ in $S$
    \ENDFOR
    \STATE return $S$
   \end{algorithmic}
 \vskip1pt
\end{algorithm}	

\Paragraph{k-Center (kCen)} 
This method creates a $k$-center clustering of $P_x$ using the greedy Gonzalez algorithms~\cite{Gon85}; that finds a set of $k$ center points which (with a factor $2$) minimizes the distance to the furthest data point.   
This is inspired by both a recent way to approximate the kernel mean (equivalent to the KDE)~\cite{arnold2015sparse} and also the initial step in (improved) fast Gauss transforms~\cite{Yang2003}.  
It takes $O(kn)$ time to find the center set, and then data points can be aggregated to the closest center in as much time.

\Paragraph{Sorting-based approaches}
For $P_x \subset \mathbb{R}^1$, these methods just sort the points, and choose $S$ as evenly spaced points in the sorted order.  Inspired by a coreset for KDEs, we can extend to higher dimensions using the Z-order space-filling curve to implicitly define a single ordering over the data points which attempts to preserve spatial locality.  Hence we refer to it as \textbf{Z-order (Z)}.  Also, inspired by this approach we take a random point from each block in the sorted order instead of the first or last of each block deterministically.  

As an extension, we propose \textbf{Z-Aggregate (ZA)} which is more careful on how it represents each interval.   It again sorts the $x$-coordinate(s) of $P$ by Z-order, and then for a set of consecutive points $P_i$ of size $h$,  $[h(i-1),hi)$ for $ i = 1,2, \ldots, k$, choose $s_x = \frac{1}{|P_i|}\sum_{p \in P_i} p_x$ and $s_y = \frac{1}{|P_i|}\sum_{p \in P_i} p_y$ as the $i$th point in $S$.

\Paragraph{Grid-based approaches}
Define a grid $G_\gamma$ into square grid cells (intervals for $P_x \subset \b{R}$) of side length $\gamma$.  
It will be convenient to designate $G_\gamma(P)$ as the non-empty grid cells, and $G_\gamma(\bar P)$ as the empty grid cells. For a grid $g$, define $P_g \subset P = \{p \in P \mid p_x \in g\}$, the points in $g$.  
In the basic method \textbf{Grid (G)}, for each $g \in G_\gamma(P)$, randomly place one point from $P_g$ into $S$, and give it a weight $|P_g|$.  

In an extension \textbf{G-Aggregate (GA)}, for each $g \in G_\gamma(P)$ we create a new point to place in $S$ as $(s_x, s_y)$ defined $s_x = \frac{1}{|P_g|} \sum_{p \in P_g} p_x$ and $s_y = \frac{1}{|P_g|} \sum_{p \in P_g} p_y$.  

The above algorithms can be subtly further improved by adding extra points in the empty grids with non-empty grids as neighbors, we call this \textbf{Aggregate-Neighbor (AN)}. Specifically these empty, but adjacent cells generate a point at the cell center $c$ with value equal to the kernel regression value $\reg_P(c)$.  This takes a bit longer than just performing an aggregate, but these empty, but adjacent cells are few so the time burden is negligible.  This is inspired by the work \cite{DBLP:journals/corr/abs-1305-3207} and the illustrative toy example in Figure \ref{fig:syn_err}. 
We will see the improvement is especially significant for $P_x \subset \b{R}^2$.

\subsection{Progressive Grid-based approaches}
In many scenarios, $P_x \subset \b{R}$ and this coordinate represents time.  Let the current time $t_{\textsc{now}}: x = 0$, and so all other values are negative (say $5$ hours ago is $x = -5$).  
In these settings, we might only examine windows of the data over $x \in [-T, 0]$, that is including now, and up to $T$ time units into the past.  Further we can assume over any view we would set the bandwidth $\sigma$ so that $\Delta = \max_{p, p' \in P} \|p_x - p'_x\|/\sigma = T/\sigma$ is upper bounded; otherwise the smoothing is below the resolution of the what can fit in a view window (its too noisy).  

For these scenarios, we design a progressive approach where we allow more errors for older data points.  Extending the grid-based approaches, as data becomes older (new points arrive) we increase the grid resolution $\gamma$, and further compress the data.  Specifically, we divide $P$ into regions $R_1, \ldots, R_r$ so the resolution $\gamma_i$ used in region $R_i$ is $\gamma_i = a^{i-1} \gamma_1$, where $a$ is a constant (we use $a=1.5$ in our experiments).  Setting the width of region $\mathsf{width}(R_i) = a^{i-1} \mathsf{width}(R_1)$ ensures that there are the same number of grid cells in each region.  Then for a fixed resolution in the first region, the size of the coreset will grow only logarithmically with time. 

\section{Analysis}
\label{sec:analysis} 

We start by providing some structural lemmas that relate approximations of kernel density estimates and weighted kernel density estimates to kernel regression.  Then we will use these results to bound the accuracy of specific techniques.  

Our goal in each case is to show that the coreset $S$ approximates the full data set $P$ in the following sense for parameters $\rho,\eps \in  (0,1)$.  
For any $q \in \b{R}^d$ such that $\kde_P(q) >\rho$, then 
\vspace{-.05in}
\[
|\reg_P(q) - \reg_S(q)| \leq \eps M,
\]
where $M = \max_{p,p' \in P} |p_y - p'_y|$.  
Then call $S$ an \emph{$(\rho,\eps)$-coreset} of $P$.  

We believe such strong worst case bounds should be surprising.  If we revisit Figure \ref{fig:syn_err} we can observe that removing one point can cause error in $\reg_P(q) - \reg_S(q)$ on the order of $M$ (in this case $M/4$). 

\begin{figure}[t!]
  \includegraphics[width=.49\linewidth]{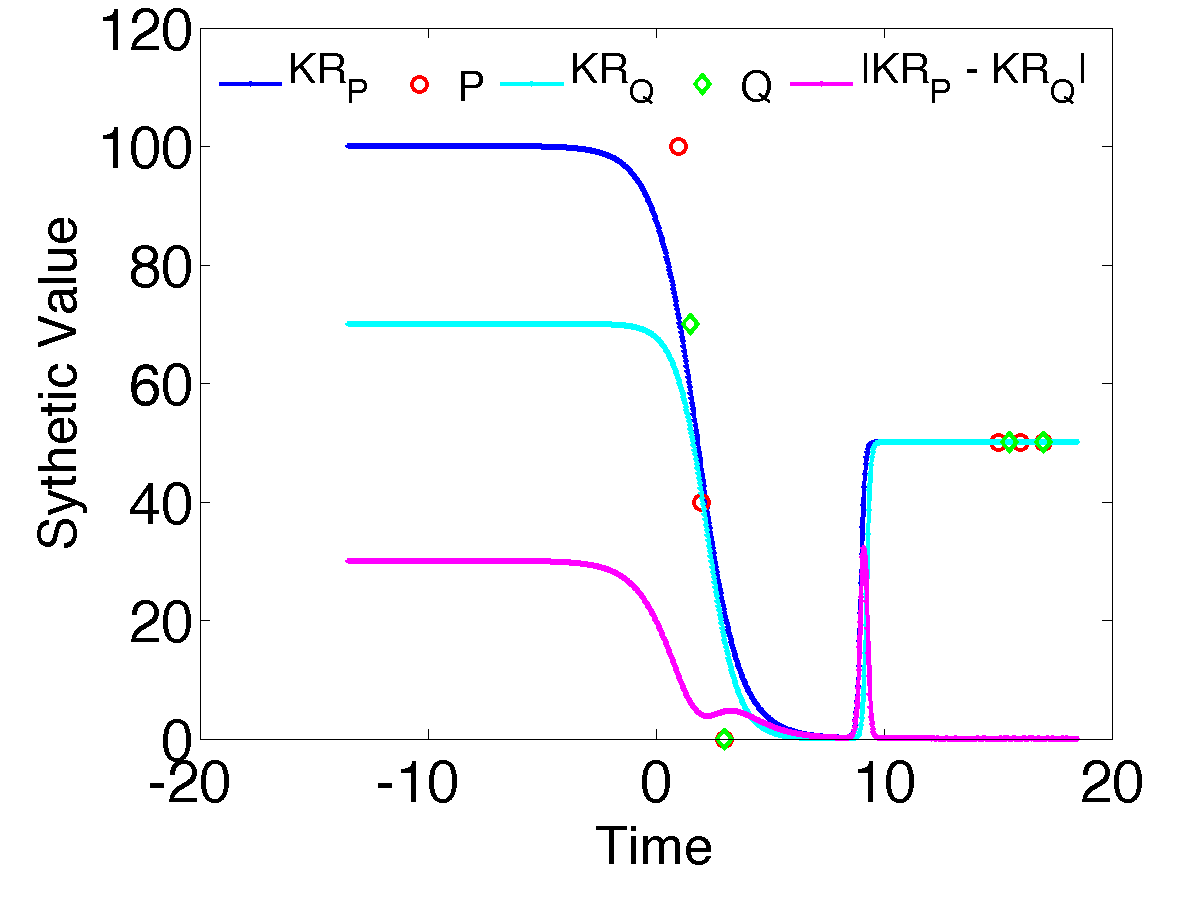}
  \includegraphics[width=.49\linewidth]{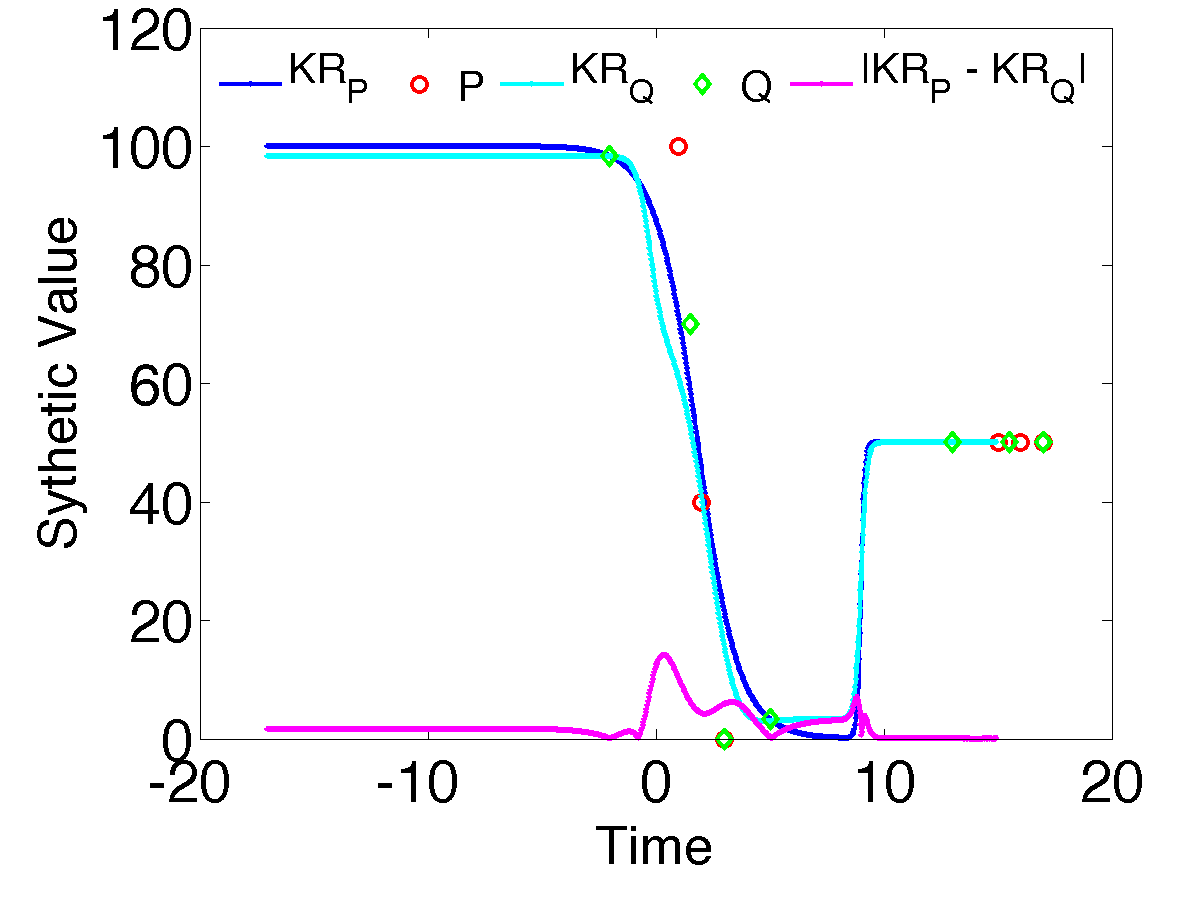}
\vspace{-.2in}
  \caption{\small \sffamily Example improvement of \textbf{Aggregate-Neighbor} (right) over \textbf{G-Aggregate} (left).  Input $P = \{(1\ 100),(2\ 40),(3\ 0),(15\ 50),(16\ 50),(17\ 50)\}$. With \textbf{G-Aggregate} with $\gamma = 2$, the coreset $Q = \{(1.5\ 70), (3\ 0), (15.5\ 50),(17\ 50)\}$. The largest errors occur at $x<0$ and around $x = 9$. 
If we add the extra points at the empty grid cells with non-empty neighbor grids, i.e., \textbf{Aggregate-Neighbor}, then $L_\infty$ is significantly reduced.  We add three points $(-2.06\ 98.3124), (5\ 3.2559)$ and $(13\ 50)$.} 
\label{fig:syn_err}
\end{figure}
 
\vspace{-.05in}
\Paragraph{Structural results}
We need a few definitions and previous results before we can begin stating our new structural tools.  
A data set $X$ and a family of subsets $\c{R}$ define a \emph{range space} $(X,\c{R})$, and the range space's VC-dimension $\nu$~\cite{VC71} (informally) describes the combinatorial complexity of the ranges; typically $\nu = \Theta(d)$.  A kernel $K : \b{R}^d \times \b{R}^d \to \b{R}^+$ is \emph{linked} to a range space if for every threshold $\tau$ and subset $Y_x \subset X$ defined $Y_x = \{y \in X \mid K(x,y) \geq \tau,\; x \in \b{R}^d\}$ there exists a range $R \in \c{R}$ such that $R \cap X = Y_x$.  
Importantly, all centrally symmetric kernels (including Gaussians) are linked to a range space where $\c{R} = \c{B}$, meaning all subsets are defined by inclusion in balls.  

A \emph{relative $(\rho,\eps)$-approximation} of $(X,\c{R})$ is a set $Y$
\[
\max_{R \in \c{R}} \left| \frac{|R \cap X|}{|X|} - \frac{|R \cap Y|}{|Y|} \right| \leq \eps \max \left\{\frac{|R \cap X|}{|X|},\rho \right\}.  
\]
Similarly, define a \emph{relative $(\rho,\eps)$-approximation} of $(P,K)$ for kernel $K$ as a set $S$ such that 
\[
\max_{x \in \b{R}^d} \left|\kde_P(x) - \kde_S(x)\right| \leq \eps \max\{ \kde_P(x),\rho\}.
\]
Define a \emph{$(\rho,\eps)$-approximation for kernel regression of $P$} as a set $S$ such that $\kde_P(q) \geq \rho$, then
\[
|\reg_S(q) - \reg_P(q)| \le \eps M, 
\]
where $M = \max_{p,p' \in P} |p_y - p'_y|$.
Define a (non-relative) $\eps$-approximation $Y$ of a range space $(X,\c{R})$, so
\[
\max_{R \in \c{R}} \left| \frac{|R \cap X|}{|X|} - \frac{|R \cap Y|}{|Y|} \right| \leq \eps.  
\]

It is know an $\eps$-approximation can be constructed, with probability at least $1-\delta$ via a random sample $S$ of size $O((1/\eps^2) (\nu + \log 1/\delta))$, and a relative $(\rho, \eps)$-approximation with size $O((1/\rho \eps^2) (\nu \log(1/\rho) + \log(1/\delta)))$~\cite{LLS01,har2011geometric}.  Given an $\eps$-approximation $S$ of a range space linked to $K$, then it is known~\cite{JoshiKommarajuPhillips2011} that it is also a (non-relative) $\eps$-approximation of $(P,K)$.  
In this paper (in the Appendix \ref{ap:link}) 
we generalize this linking result (roughly following the structure of the proof in \cite{JoshiKommarajuPhillips2011}) to relative $(\rho,\eps)$-approximations.  

\begin{theorem}
\label{thm:link}
For any kernel $K : \b{R}^d \times \b{R}^d \to \b{R}^+$ linked to a range space ($\b{R}^d, \c{A}$), a relative $(\rho,\eps)$-approximation $S$ of $(P, \c{A})$ is a $(\rho K^+,2\eps)$-approximation of $(P, K)$, where $K^+ = \max_{p,q\in P} K(p,q)$. 
\end{theorem}

Next we provide a sufficient condition for $(\rho,\eps)$-approximation for kernel regression.  

\begin{lemma}
\label{lm:kr}
For error parameters $\alpha, \beta, \rho > 0$, with $\alpha \leq 1/2$, consider a point set $P \subset \b{R}^{d+1}$. 
Let $S$ be a coreset of $P$ so that for any query point $q \in \b{R}^d$,  both
\begin{align*} 
|\kde_P(q) - \kde_S(q)| & < \alpha \max\{\kde_P(q), \rho\}
\\
|\wkde_P(q)-\wkde_S(q)| & < \beta M.
\end{align*}  
Then for any $q \in \b{R}^d$ such that $\kde_P(q) \geq \rho$, then 
$
|\reg_S(q) - \reg_P(q)| \le 4 (\alpha + \beta/\rho) M.
$
\end{lemma}
\begin{proof}
Change the units of $p_y$ so all values lie between $1$ and $2$.  The shifting of these values does not change the approximation factor $\beta M$, but the rescaling of the range changes the bound to $| \wkde_P(q) - \wkde_S(q) | \leq \beta$, and also ensures $1 \leq \reg_P(q) \leq 2$.  And recall, $p_y$ values have no bearing on $\kde_P(q)$.  

By using the Gaussian kernel we have $\kde_S(q) > 0$ and also $0 \leq \wkde_P(q) \leq 2$.  Thus we can consider relative error bounds, using $\kde_P(q) > \rho$, and hence also $\wkde_P(q) > \rho$.  
\begin{align*}
\frac{\reg_S(q)}{\reg_P(q)}
& = 
\frac{\wkde_S(q)}{\wkde_P(q)} \frac{\kde_P(q)}{\kde_S(q)}
 \\& \geq 
\left(1 - \frac{\beta}{\wkde_P(q)} \right)
\left(1 - \frac{\alpha}{1+\alpha} \right) 
\\ & = 
1 - \frac{\beta}{\wkde_P(q)} - \frac{\alpha}{1+\alpha} + \frac{\alpha \beta}{(1+\alpha)\wkde_P(q)}
 \\& \geq 
1 - \frac{\beta}{\rho} - \alpha
\end{align*}

Next we see the relative error bound is slightly different in the other direction.  

\begin{align*}
\frac{\reg_S(q)}{\reg_P(q)}
& = 
\frac{\wkde_S(q)}{\wkde_P(q)} \frac{\kde_P(q)}{\kde_S(q)}
\\& \leq 
\left(1 + \frac{\beta}{\wkde_P(q)} \right)
\left(1 + \frac{\alpha}{1-\alpha} \right)
\\ & = 
1 + \frac{\beta}{\wkde_P(q)} + \frac{\alpha}{1-\alpha} + \frac{\alpha \beta}{(1-\alpha)\wkde_P(q)}
\\&  \leq 
1 + \frac{\beta}{\rho} + \frac{\alpha}{1-\alpha} + \frac{\alpha \beta}{(1-\alpha)\rho}
\\& = 
1+ \frac{\beta}{(1-\alpha)\rho}+ \frac{\alpha}{1-\alpha}
\\& \leq  
1+\frac{2\beta}{\rho} + 2\alpha
\end{align*}

Together these imply $\frac{\reg_S(q)}{\reg_P(q)} \in [1- \beta/\rho - \alpha, 1 + 2\beta/\rho + 2\alpha]$.  
This translates to the following additive error
\begin{align*}
|\reg_P(q) - \reg_S(q)| 
& \leq 
2(\beta/\rho + \alpha) \reg_P(q)
\\& \leq
2 (\beta/\rho + \alpha)  2M
 =
4 (\alpha + \beta/\rho) M.  \qedhere
\end{align*}
\end{proof}
 
We also need another property about the slope of the Gaussian kernel.  This is the only bound specific to the Gaussian kernel, so for any other kernels with a similar bound (e.g., Triangle, Epanechnikov) the remaining analysis and algorithms can apply.  
\begin{lemma}
\label{lm:glip}
A unit Gaussian kernel $K(x) = \exp(-x^2/2\sigma^2)$ is $1/\sigma$-Lipschitz.
\end{lemma}
\begin{proof}
By taking the first derivative of $K$ with respect to $x$, we have
$
\frac{dK(x)}{dx} = \exp(-\frac{x^2}{2\sigma^2})(-\frac{x}{\sigma^2} ).
$
Take the second derivative 
$
\frac{d^2K(x)}{dx^2} = \exp(-\frac{x^2}{2\sigma^2})(\frac{x^2}{\sigma^4}  - \frac{1}{\sigma^2})
$
and set $\frac{d^2K(x)}{dx^2} = 0$.   We get $x = \pm \sigma$, and thus $|\frac{dK(x)}{dx}| $ has the maximum values on $x = \pm \sigma$, equals to $  \exp(-\frac{1}{2})(\frac{1}{\sigma}) \le 1/\sigma $. So a unit Gaussian kernel is $1/\sigma$-Lipschitz. 
\end{proof}

\subsection{Accuracy of Random Sampling}
We start by analyzing how kernel regression is preserved under random sampling.  In many cases the ``input'' data to a problem should actually be modeled as a random sample of some much larger set, 
or it may be done as a first pass on data to reduce its complexity.  

The key structural result will be on sampling weighted sets.   

\begin{lemma}
For a weighted point set $(P,w)$ with $P \subset \b{R}^d$ of size $n$, then a random sample of points $Q \subset P$ of size $s = O((1/\eps^2) (d+  \log(1/\delta)))$, with probability at least $1-\delta$, satisfies for any $B \in \c{B}$
\[
\left| \frac{1}{n} \sum_{p \in P \cap B} w(p) - \frac{1}{s} \sum_{p \in Q \cap B} w(p) \right| \leq \eps M,
\]
where $M = \max_{p \in X} w(p) - \min_{p \in X} w(p)$.  
\label{lem:weighted-sample}
\end{lemma}
\begin{proof}
First assume $\max_{p \in X} w(p) = 1$ and that $\min_{p \in X}  w(p) = 0$; then $M = 1$.  
Otherwise, we can simply ``change the units" by uniformly shifting and scaling all $w$ values to reach this scenario.  

We first consider $(X,w)$ as a point set $P \subset \b{R}^{d+1}$, where the $y$-coordinate is $w(p)$.  Then we consider the range space $(P,\c{R})$ where $\c{R}$ defines the set of subsets induced by ranges which are balls in the first $d$ coordinates, and an interval in the $y$-coordinate; we refer to them as hypercylinders.  The range space has VC-dimension $O(d)$.  
For a given query $B \in \c{B}$ on $X$ ($B$ is the ball in $\b{R}^d$ on the $x$-coordinates), we are interested hypercylinders $R \in \c{R}$ so that the $x$-coordinates are restricted to those in our query choice of $B$.  

In fact, we can break the hypercylinder $R$, which implicitly has a $y$-interval of $[0,1]$, up into $\eta = c/\eps$ disjoint hypercylinders (design constant $c$ so that $c/\eps$ is an integer), each with the same ball $B$ in $x$-coordinates and a $y$ width of $\eps/c$.  Let $P_i$ be the set $P$ restricted to $i$th such $y$ interval.  We can round all values within the interval to a value $v_i = i \cdot (c/\eps)$, incurring at most $\eps/c$ error.  
Then if each $i$th piece's sample $Q_i$ is off in count by $\alpha_i$ and $|\sum_i \alpha_i |\leq \eps n/2$, then we can say the total error is at most $|P| \eps/c + n \eps/2$.  Setting $c \geq 2$, ensures the total as is at most $\eps n$ as desired.  

However, \emph{individually} bounding each $\alpha_i$ to be small is hard.  If there are few points in one of the levels, then we get a poor estimate on the count in $Q_i$ using standard techniques.  
Instead we can bound $\sum_i \alpha_i$ \emph{in aggregate}.  
By the definition of $\eps$-samples, if $Q$ is an $\eps/2$-sample of $(P,\c{R})$, then $|\sum_{r = i}^j \alpha_i| \leq \eps/2 \cdot n$ for all $i,j \in [1,\eta]$.  And this holds by our random sample with probability at least $1-\delta$.  

Now we can write the total error from $(P,w)$ to $(Q,w)$ in an ball $B \in \c{B}$ as
\begin{align*}
\frac{1}{n} \sum_{p \in P \cap I} w(p)
 &=
\frac{1}{n} \sum_{i=1}^\eta \sum_{p \in P_i \cap B} w(p)
\\ &\leq
\frac{1}{n} \sum_{i=1}^\eta \sum_{p \in P_i \cap B} \left(v_i  + \eps /c\right)
\\ &=
\frac{\eps}{c} + \frac{1}{n} \sum_{i=1}^\eta v_i |P_i \cap B|
\\ & \leq
\frac{\eps}{c}  + \frac{1}{n} \sum_{i=1}^\eta v_i \left(\alpha_i + \frac{n}{s} |Q_i \cap B|\right)
\\ &=
\frac{\eps}{c} + \frac{1}{n} \sum_{i=1}^\eta v_i \alpha_i + \frac{1}{s} \sum_{1=1}^\eta \sum_{p \in Q_i \cap B} v_i
\\ &\leq
\frac{\eps}{c} + \frac{1}{n} \sum_{i=1}^\eta \alpha_i + \frac{1}{s} \sum_{i=1}^\eta \sum_{p \in Q_i \cap B} (w(p) + \eps/c)
\\ & \leq
\frac{2 \eps}{c} + \frac{\eps}{2} + \frac{1}{s}  \sum_{p \in Q \cap B} w(p).
\end{align*}
Setting $c \geq 4$, and repeating the argument symmetrically to show the lower bound, we obtain that for any $B \in \c{B}$
\[
\left| \frac{1}{n} \sum_{p \in P \cap B} w(p) - \frac{1}{s} \sum_{p \in Q \cap B} w(p) \right| \leq \eps. \qedhere
\]
\end{proof}

This results generalizes to weighted kernel density estimates, for centrally-symmetric and non-increasing (as function of distance from center) kernels, following~\cite{JoshiKommarajuPhillips2011}.  The only change is using the weighted bound in Lemma \ref{lem:weighted-sample}, in place of where Joshi \emph{et.al.} used the unweighted bound in the definition of a ball-range space linked with the aforementioned kernels.  

\begin{theorem}
\label{thm:wkde-samp}
Consider any kernel $K : \b{R}^d \times \b{R}^d \to \b{R}^+$ linked to $(\b{R}^d, \c{B})$.  
For a weighted point set $(X,w)$ with $X \subset \b{R}^d$, then a random sample $Q \subset X$ of size $s = O((1/\eps^2) (d + \log(1/\delta)))$, with probability at least $1-\delta$, for any $x \in \b{R}^d$ satisfies 
\[
\left| \wkde_{X,w}(x) - \wkde_{Q,w}(x) \right| \leq \eps M,
\]
where $M = \max_{p \in X} w(p) - \min_{p \in X} w(p)$.  
\end{theorem}

Now we are ready to show the main result.  

\vspace{-.05in}

\begin{theorem}
Consider a point set $P \subset \b{R}^{d+1}$ of arbitrary size, and parameters $\rho,\eps \in (0,1)$.  
Let $S$ be a uniform sample from $P$ of size $O(\frac{1}{\eps^2 \rho^2}( d \log(1/\rho) + \log(2/\delta)))$, with probability at least $1-\delta$, the set $S$ is a $(\rho,\eps)$-approximation for kernel regression on $P$. 
\end{theorem}
\vspace{-.05in}
\begin{proof}
For a binary range space (such as $(P, \c{B})$) with constant VC-dimension~\cite{VC71} $\nu$, a random sample $S$ of size $k = O(\frac{1}{(\eps')^2 \rho} (\nu \log(1/\rho) + \log(2/\delta)))$ provides an $(\rho, \eps')$-sample with probability at least $1 - \delta/2$ \cite{LLS01, har2011geometric}. 
Theorem \ref{thm:link} gives a linking result for kernel density estimate, implying that this is also a relative $(\rho,2\eps')$-coreset for a kernel where $K(x,x) = 1$.  
This satisfies the first condition of Lemma \ref{lm:kr} with $\alpha = 2\eps'$.  

Second we invoke Theorem \ref{thm:wkde-samp} so that we have with probability at least $1-\delta/2$ that 
$
|\wkde_P(q) - \wkde_S(q)| \leq (\eps' \rho) M, 
$
hence satisfying the second condition of Lemma \ref{lm:kr} with $\beta = \eps' \rho$.  

Setting $\eps' = \eps/16$ invoking Lemma \ref{lm:kr}, then with probability at least $1-(\delta/2 +\delta/2) = 1-\delta$, for any $q \in \b{R}^d$ that 
$
|\reg_S(q) - \reg_P(q)| \leq 4 (\eps/8 + (\eps \rho / 16)/\rho) M \leq \eps M. 
$
\end{proof}

\subsection{Accuracy of Grid-Based Approaches}
We first bound the error in \textbf{Grid}.  This implies other related algorithms (\textbf{G-Aggregate}, \textbf{Aggregate-Neighbor}) will have have the same asymptotic error bounds for $d$ constant.  

\begin{theorem}
\textbf{Grid} with $\gamma = \frac{\eps \sigma \rho}{8 \sqrt{d}}$ produces $S$, a $(\rho, \eps)$-coreset for the kernel regression of $P \subset \mathbb{R}^{d+1}$.  
\end{theorem}
\begin{proof} 
We will prove bounds on both error in $\kde_P$ and $\wkde_P$ separately, then combine them with Lemma \ref{lm:kr}.  
This algorithm maps all points $P_g$ for a grid cell $g \in G_\gamma$ to a single point, and by reweighting, changes each points location by at most $\gamma \sqrt{d}$.  Using that $K$ is $(1/\sigma)$-Lipschitz, this changes $\kde_P$ by at most $\gamma \sqrt{d} / \sigma$ in $\kde_S$.  Only considering $\kde_P(q) \geq \rho$, then 
$|\kde_P(q) - \kde_S(q)| \leq \frac{\gamma \sqrt{d}}{\rho \sigma} \max \{\rho, \kde_P(q)\}$.  

For $\wkde_P$ the analysis is similar, but we may also replace $p_y$ for a point $p \in P_g$ with a different $s_y$.  We can bound $|p_y - s_y| \leq M = \max_{p,p' \in P} |p_y - p'_y|$.  Hence for all $q \in \mathbb{R}^d$, then $|\wkde_P(q) - \wkde_S(q)| \leq \gamma \sqrt{d} M/ \sigma$.  

Combining these two bounds together with Lemma \ref{lm:kr} we obtain (for $q$ with $\kde_P(q) \geq \rho$) that 
$
|\reg_P(q) - \reg_S(q)| 
\leq 
4 (\frac{\gamma \sqrt{d}}{\rho \sigma} + \frac{\gamma \sqrt{d}}{\sigma}/\rho) M 
=
4 (\frac{\eps}{8} + \frac{\eps \rho}{8}/\rho) M
= 
\eps M.  
$
\end{proof}

\noindent
By aggregating on each relevant grid cell, we bound coreset size with $\Delta = \max_{p,p'\in P} \|p_x - p'_x\|/\sigma$.  

\begin{corollary}
For $P \subset \mathbb{R}^{d+1}$ for constant $d$, methods \textbf{Grid}, \textbf{G-Aggregate}, and \textbf{Aggregate-Neighbor}, run in $O(|P|)$ time, and return $S$ a $(\rho, \eps)$-coreset for kernel regression of $P$ of size at most $O((\Delta/\eps \rho)^d)$.  
\end{corollary}

\Paragraph{Accuracy of progressive grid-based methods}
If the size $\mathsf{width}(R_1)$ of the first region in the progressive methods is a constant, there are at most $O(\log \Delta)$ regions.  Set $\gamma = \eps \sigma \rho /8 \cdot a^{i-1}$, so each region has a grid with $O(1/\eps \rho)$ cells.  

\begin{corollary}
For $P \subset \mathbb{R}^2$, under any allowable view window of size $T$ and scaling so $T/\sigma$ is fixed, then the progressive Grid approach achieves an $(\rho, \eps)$-coreset for kernel regression of $P$ of size at most $O((1/\eps \rho) \log \Delta)$.  
\end{corollary}   

\Paragraph{Accuracy bounds for other methods}
Despite bounds for $|\kde_P(q) - \kde_S(q)|$ for other methods (e.g., \textbf{Z-order}) we are not able to show these for $\wkde_P$ and hence $\reg_P$.  For example, there exists simple examples with wildly varying density where the Z-order techniques approximates the density well, but not the function values.  Hence, we either cannot bound the accuracy or the size for this method.  

\section{Experiments}
\label{sec:exp}
Here we run an extensive set of experiments to validate our methods. We compare $\reg_P$ where $P_x \subset \b{R}^{1}$, $P_x \subset \b{R}^{2}$, and $P_x \subset \b{R}^6$ with kernel regression under smaller coreset $\reg_S$ for both synthetic and real data. To show our methods work well in large data sets, we use large real data set ($n = 2$ million and $24$ million) and synthetic data ($n = 1$ million) for $P_x \subset \b{R}^{1}$, and real data set ($n = 1$ million) for $P_x \subset \b{R}^{2}$.  Our algorithms scale well beyond these sizes, but evaluating error was prohibitive.  

\subsection{Data Sets}
For real data we consider ``Individual Household Electric Power Consumption" data set on UCI Machine Learning Repository. The number of instances is $2{,}075{,}259$, we use the first three attributes to do kernel regression. Date, time (together for $x$-value), and global active power (for $y$-value): household global minute-averaged active power (in kilowatt). 
This data set has gaps on the $x$-axis, and kernel regression does a nice job of interpolating those gaps. 

To demonstrate the effectiveness of progressive grids, we use a "CloudLab" dataset. CloudLab~\cite{ricci2014introducing} is cloud computing platform, and we have obtained a trace of power usage from the Utah site with $400$ million values.  We use the most recent $10$-month window which has size $24{,}351{,}363$.   

 \begin{figure*}[t!]
 \includegraphics[width=0.3\linewidth]{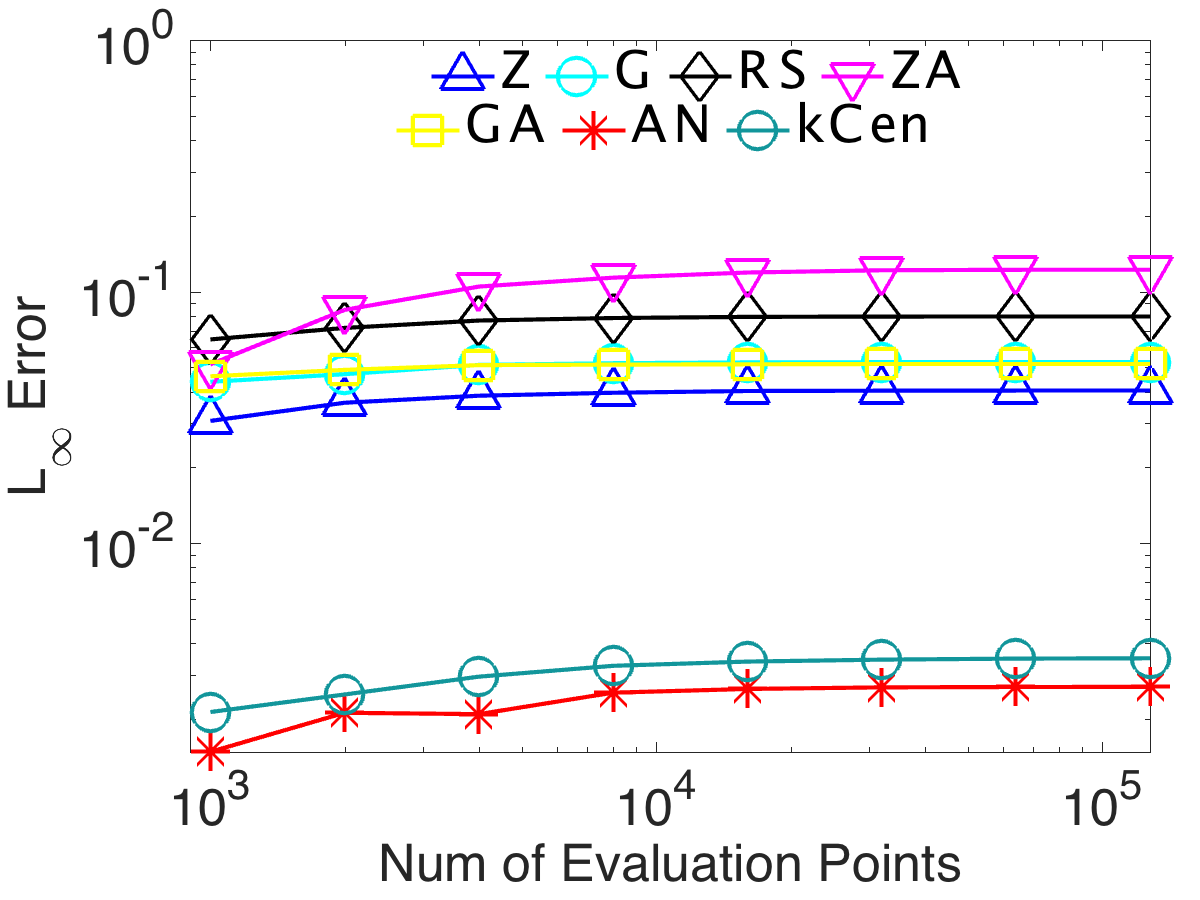}\hfill
 \includegraphics[width=0.3\linewidth]{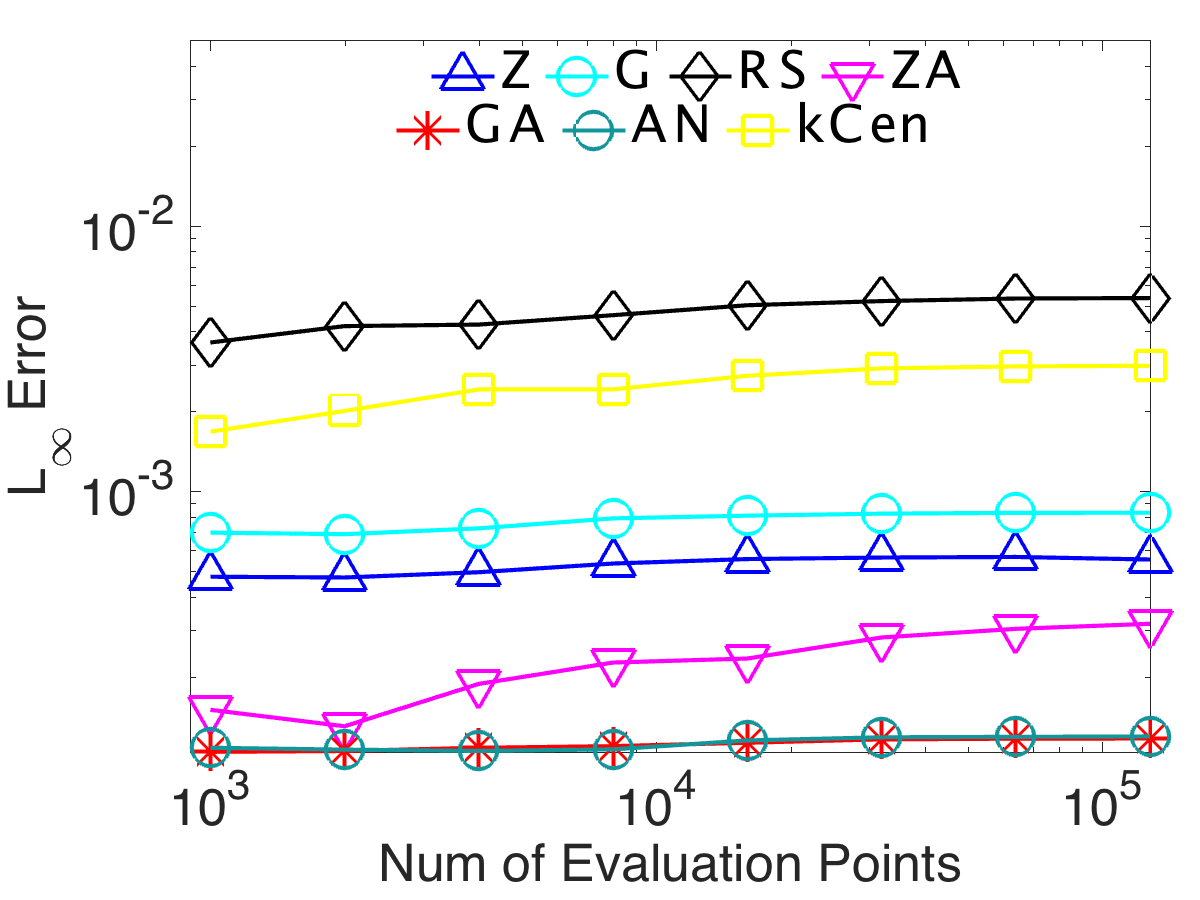}\hfill
 \includegraphics[width=0.3\linewidth]{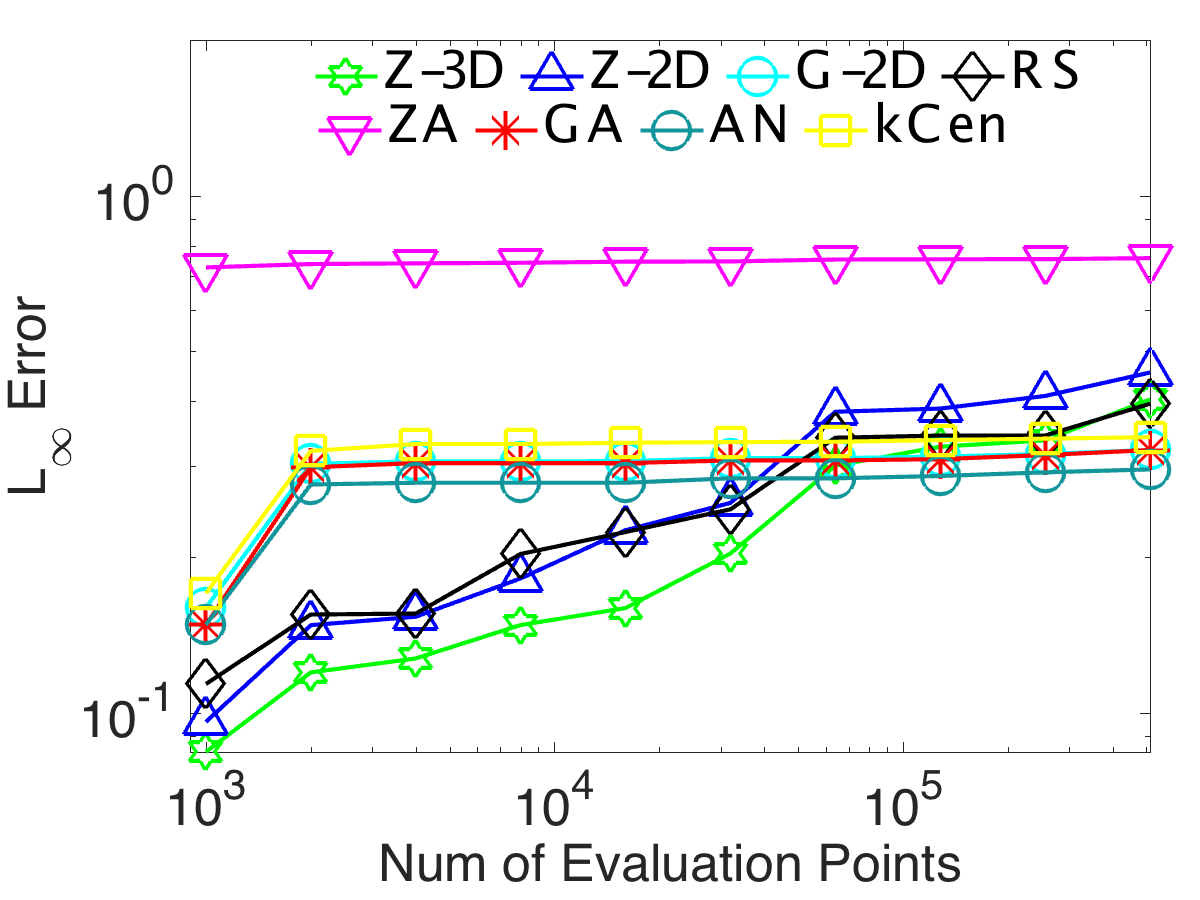}
\vspace{-.15in}
  \caption{   \label{fig:eval_size}
  \sffamily The maximum $L_\infty$ error found based on the number of evaluation points on real (left) and synthetic data(middle) when $P_x \subset \b{R}^{1}$, and real data(right) when $P_x \subset \b{R}^{2}$.  }
\end{figure*}

 \begin{figure*}[h!]
\vspace{-.1in}
 \includegraphics[width=0.33\linewidth]{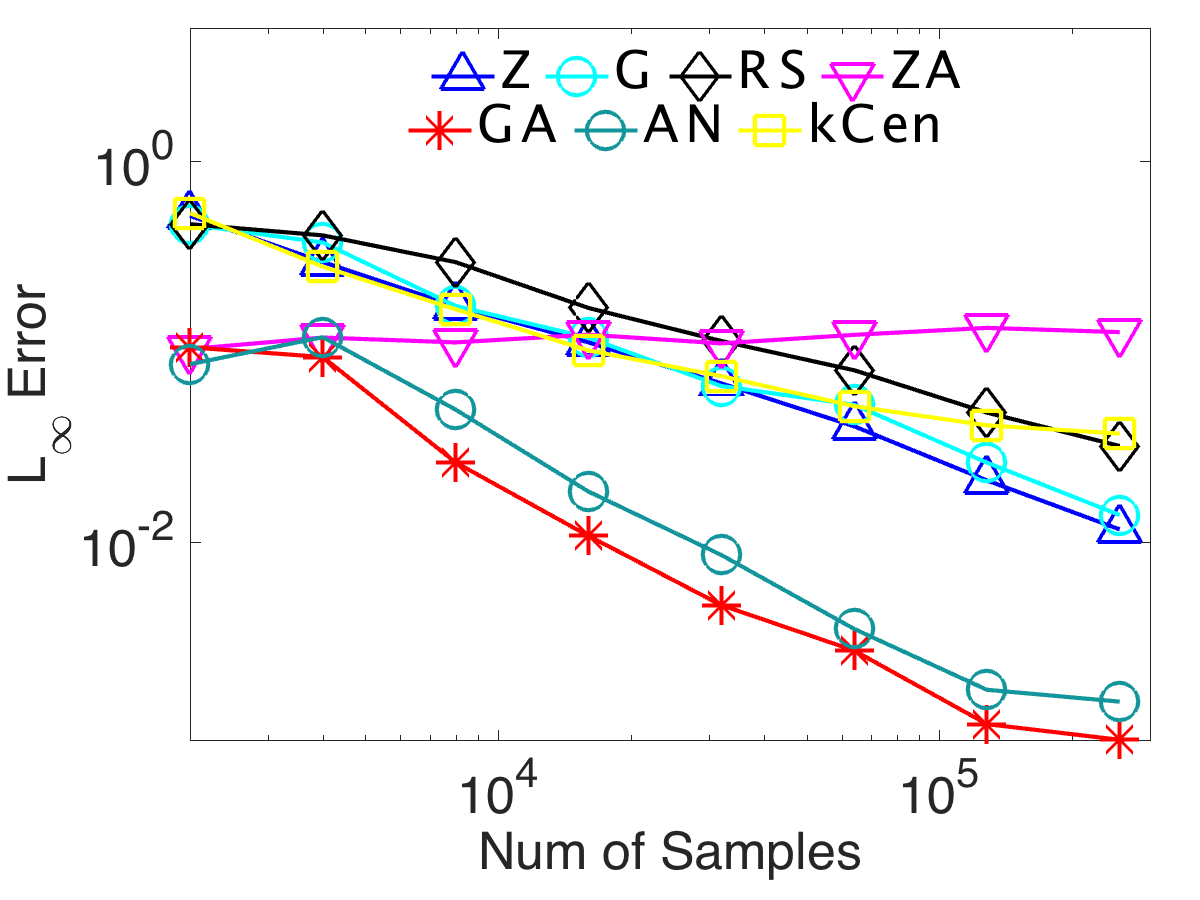}
  \includegraphics[width=0.33\linewidth]{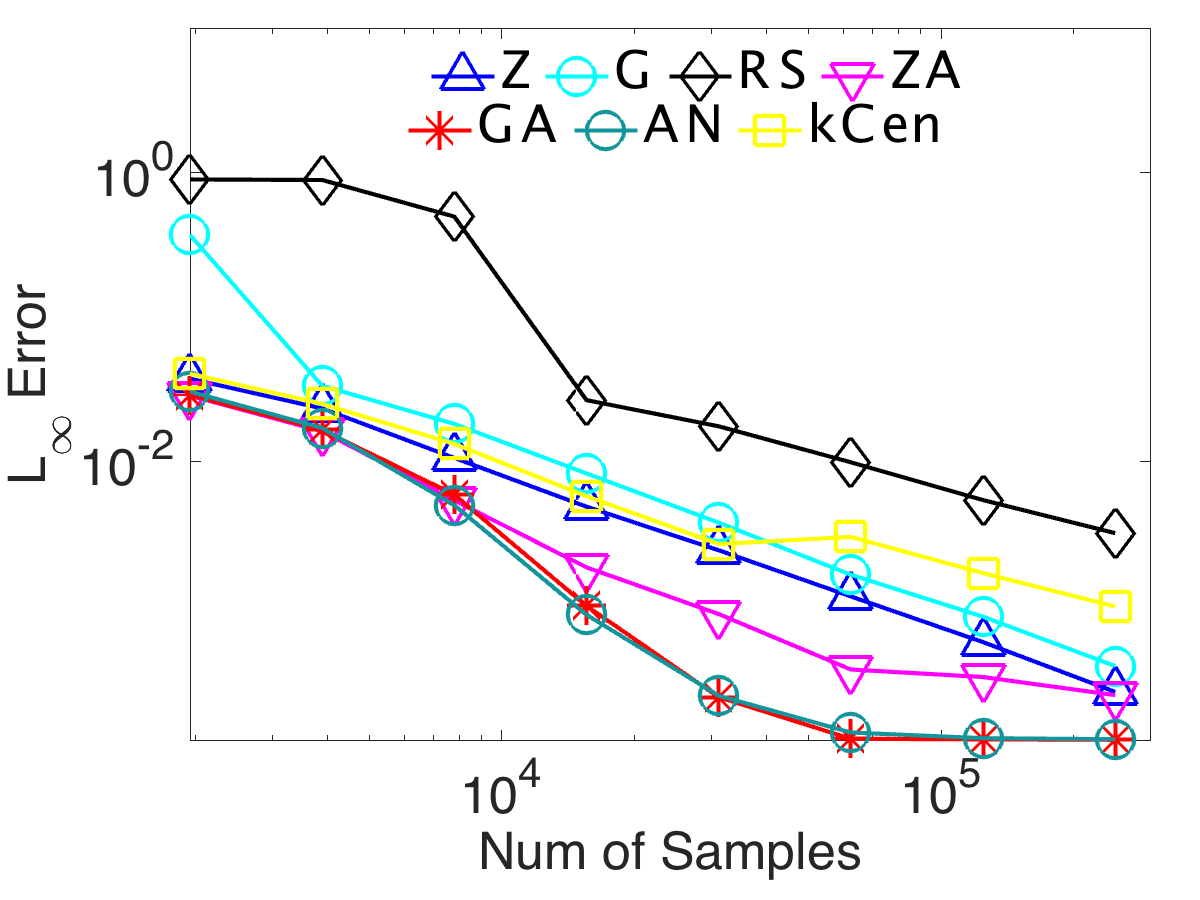}
  \includegraphics[width=0.33\linewidth]{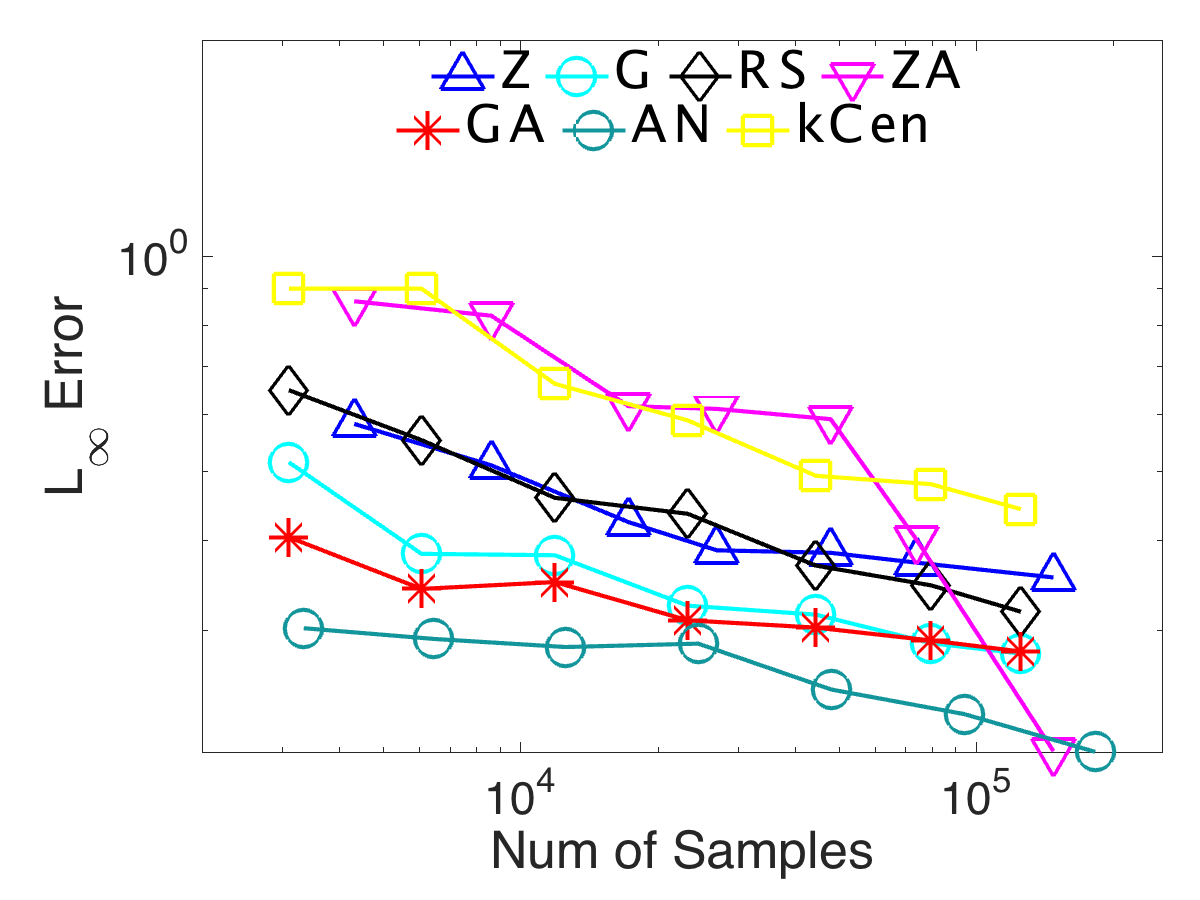}
  \vspace{-.3in}
  \caption{\label{fig:coreset}
  \sffamily $L_\infty$ error for coresets when also testing sparse regions on real data(left) and synthetic data(middle) when $P_x \subset \b{R}$, and real data(right) when $P_x \subset \b{R}^{2}$.}
  \vspace{-.1in}
\end{figure*}

The time series synthetic data $P_x \subset \b{R}^{1}$ is generated using formula:
$
	y_i = c+ \phi y_{i-1} + N(0,\sigma),
$ 
where the $x$-coordinates are $i = 0,1, \cdots$ and $y_i$ is the corresponding $y$ coordinates. It mimics a stock price so the next data depends on the previous one plus some random noise. In the experiment, we set $c = 0, \phi = 1, y_0 = 10, \sigma = 1$ and generate $1$ million points. The first $10{,}000$ points and kernel regression with bandwidth $50$ and $200$ are shown in Figure \ref{fig:syn_vis}.  

For $P_x \subset \b{R}^{2}$ real dataset, we consider OpenStreetMap data from the state of Iowa. Specifically, we use the longitude and latitude of all highway data
points as $P_x$ and time stamp as $P_y$. Kernel regression on this dataset can give a good approximation of when the highway data point is added. 

For high dimension experiment, we consider two datasets. One is house price dataset (CAD) in StatLib \cite{pace1997sparse} and the other is Physicochemical Properties of Protein Tertiary Structure Dataset (CASP) from UCI machine learning repository. CAD dataset contains $20{,}640$ observations on housing prices with $9$ economic covariates and CASP dataset has $45{,}730$ data points for $10$ random variables.  For both datasets, we use first $6$ features to do the kernel regression. 

\subsection{Effectiveness of Coresets}

Coresets guarantee that kernel regression error is bounded for \emph{all} values of $q \in \b{R}$ (as long as the data is not too sparse).  But evaluating at all of these points, is by definition, impossible.  As a result, we evaluate over a very fine covering of evaluation points (in our case $128{,}000$ for $P_x \subset \mathbb{R}^1$ and $512{,}000$ for $P_x \subset \mathbb{R}^2$).  We have plotted error as the number of evaluation points increase and observed that all methods clearly converge well before this many samples.   

In more detail, 
we randomly generate a evaluation point $q$ in the domain $\b{R}$ for $P_x \subset \b{R}$ and $q$ in the domain $\b{R}^2$ for $P_x \subset \b{R}^2$, without the restriction $\kde_P(q) > \rho$.  
With fixed coreset size $64{,}000$, we experiment on the number of evaluation points from $1{,}000$ to $128{,}000$ for $P_x \subset \b{R}$ and $1{,}000$ to $512{,}000$ for $P_x \subset \b{R}^2$ in Figure \ref{fig:eval_size}.
As the number of evaluation points increases, the value of maximum error in the domain will consistently approach some error value  and we can then have some confidence that we have the correct worst case error as this processes plateaus.
Under all the subset selection methods (Figure \ref{fig:eval_size}), the errors are steady at size $128{,}000$ for $P_x \subset \b{R}$ and $512{,}000$ for $P_x \subset \b{R}^2$, so we use evaluation points of size $128{,}000$ for $P_x \subset \b{R}$ and $512{,}000$ for $P_x \subset \b{R}^2$  in the following experiments.

Since all the methods are randomized algorithms, we run all the subset selection methods ten times and use the average errors as the final results. The bandwidth is set to $400$ for the real dataset in $\b{R}^1$,  $50$ for the synthetic dataset in $\b{R}^1$, and $50$ for real data in $\b{R}^2$; other bandwidths have similar performance. 

Figure \ref{fig:coreset} shows all the methods converge as the size of the coreset increases.  The exception is \textbf{Z-Aggregate} on real data in $\b{R}$; on inspection, the problem occurs in sparse regions, similar to Figure \ref{fig:syn_err}. 
\textbf{G-Aggregate} and \textbf{Aggregate-Neighbor}  (and sometimes \textbf{Z-Aggregate}) work significantly better compared to all the other methods in all datasets with $P_x \subset \b{R}$.  They consistently decrease, at certain sizes have one or two orders of magnitude less error, and obtain virtually no error at size about $50{,}000$.   
Even when the size of the coreset is small, \textbf{G-Aggregate} and \textbf{Aggregate-Neighbor} have very small errors and converge very fast when the size increases.  
For $P_x \subset \b{R}^2$, \textbf{Aggregate-Neighbor} achieve noticeably smaller error, but \textbf{G-Aggregate} (and \textbf{Grid}) perform well and are simpler.  

In particular \textbf{G-Aggregate} and \textbf{Aggregate-Neighbor} stay \emph{at least} one order of magnitude smaller in error than \textbf{Random Sampling}. 
This indicates it is much better to aggregate based on $x$-value than just randomly sample.  It also justifies further thinning the data with these methods if the data should be modeled as a random sample since the additional error introduced would be negligible compared to what was already present due to the sampling.  

The grid-based methods also consistently outperform the (z-order) sorting-based methods, so it is better to compress based on the $x$-coordinate change, rather than on the number of points.  

Filling in a few neighbor values (the \textbf{-Neighbor} method) can also result in significant gains in accuracy for $P_x \subset \b{R}^2$, but for $P_x \subset \b{R}$ it does not show much improvement, and sometimes performs worse.  The larger error per points (Figure \ref{fig:coreset}, real data) is mainly due to the extra points added without much error reduction.  It seems just aggregating does a good enough job for $P_x \subset \b{R}$, but for $P_x \subset \b{R}^2$ more complicated situations arise where this extra step is helpful.

  \begin{figure}
  \includegraphics[width=0.49\linewidth]{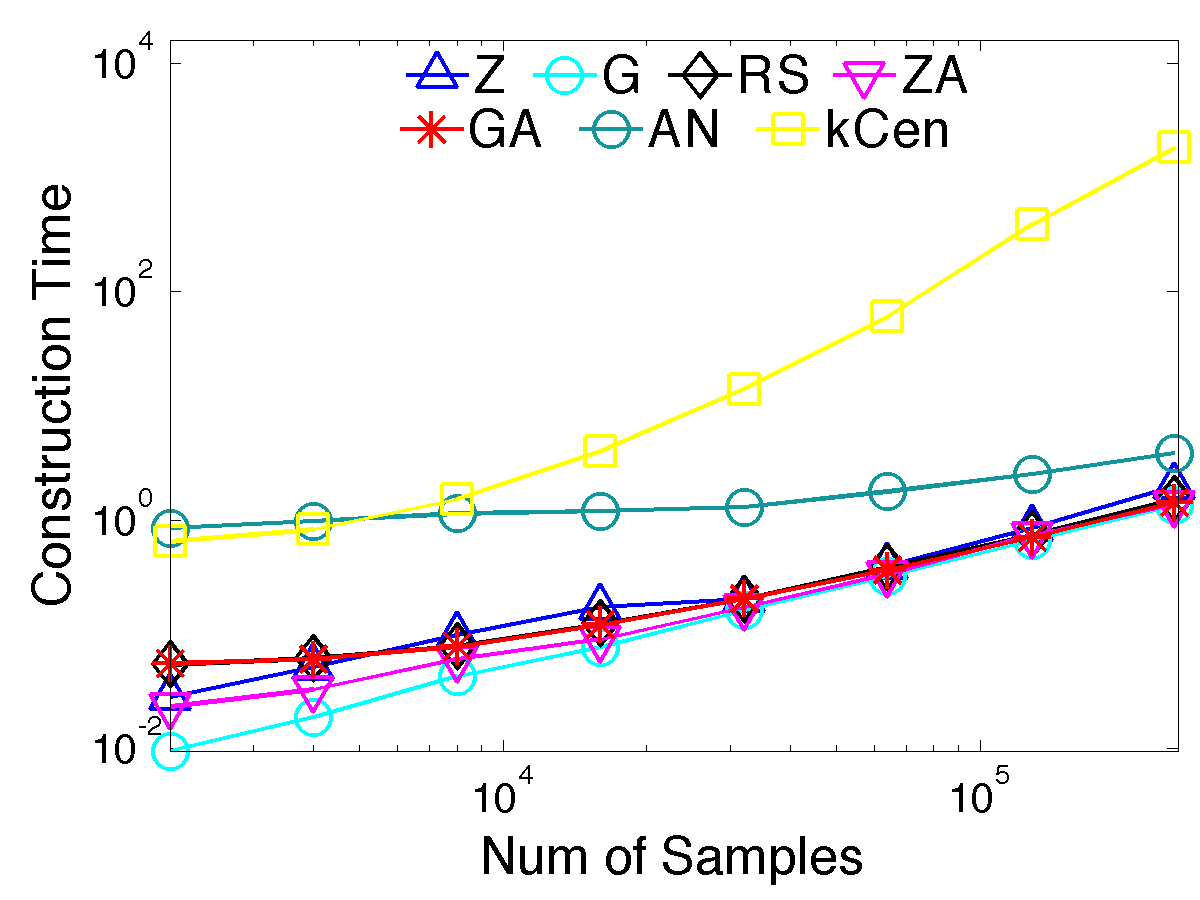}
  \includegraphics[width=0.49\linewidth]{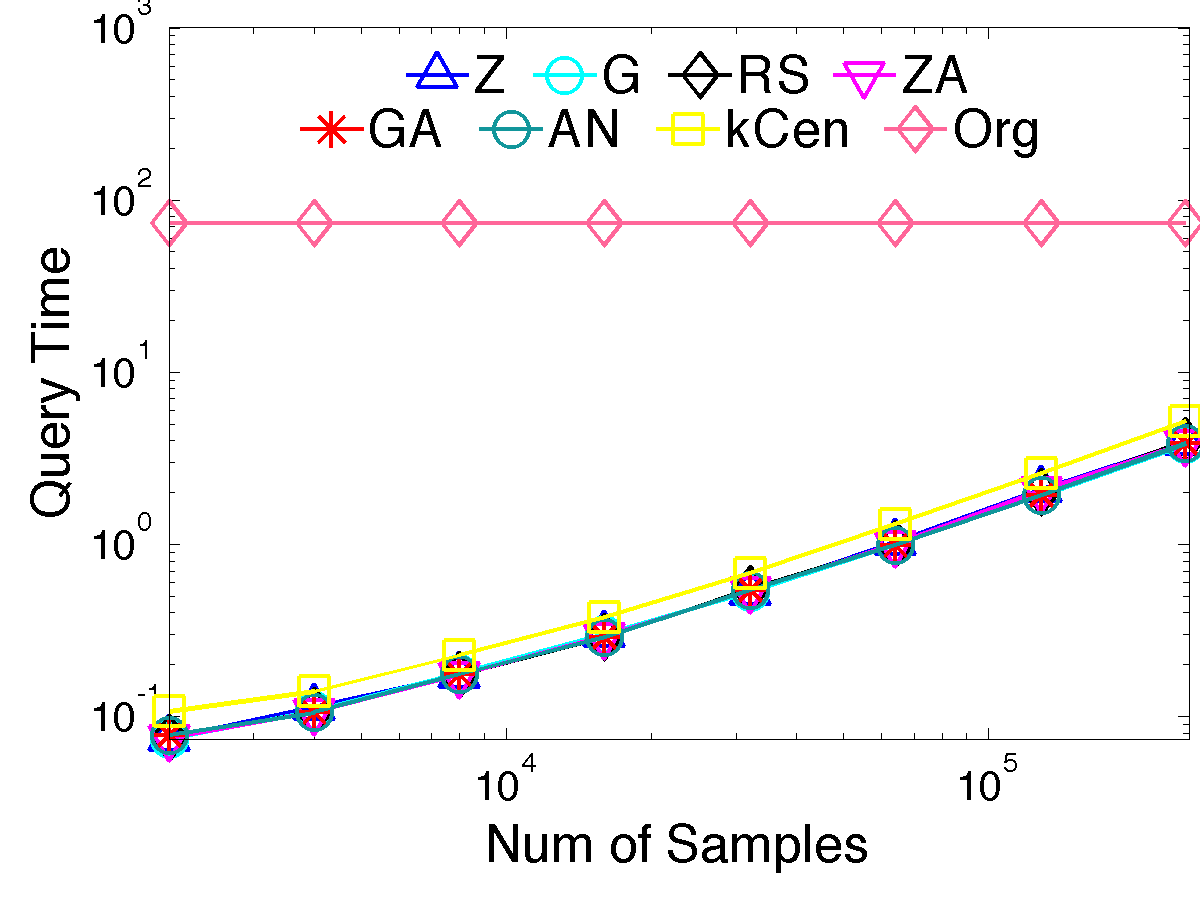}
  \vspace{-.2in}
  \caption{\sffamily Comparison of construction time and query time for real dataset with $P_x \subset \b{R}$.}
   \label{fig:housesyn_time}
   \vspace{-.05in}
\end{figure}

\begin{figure}
\vspace{-.1in}
  \includegraphics[width=0.49\linewidth]{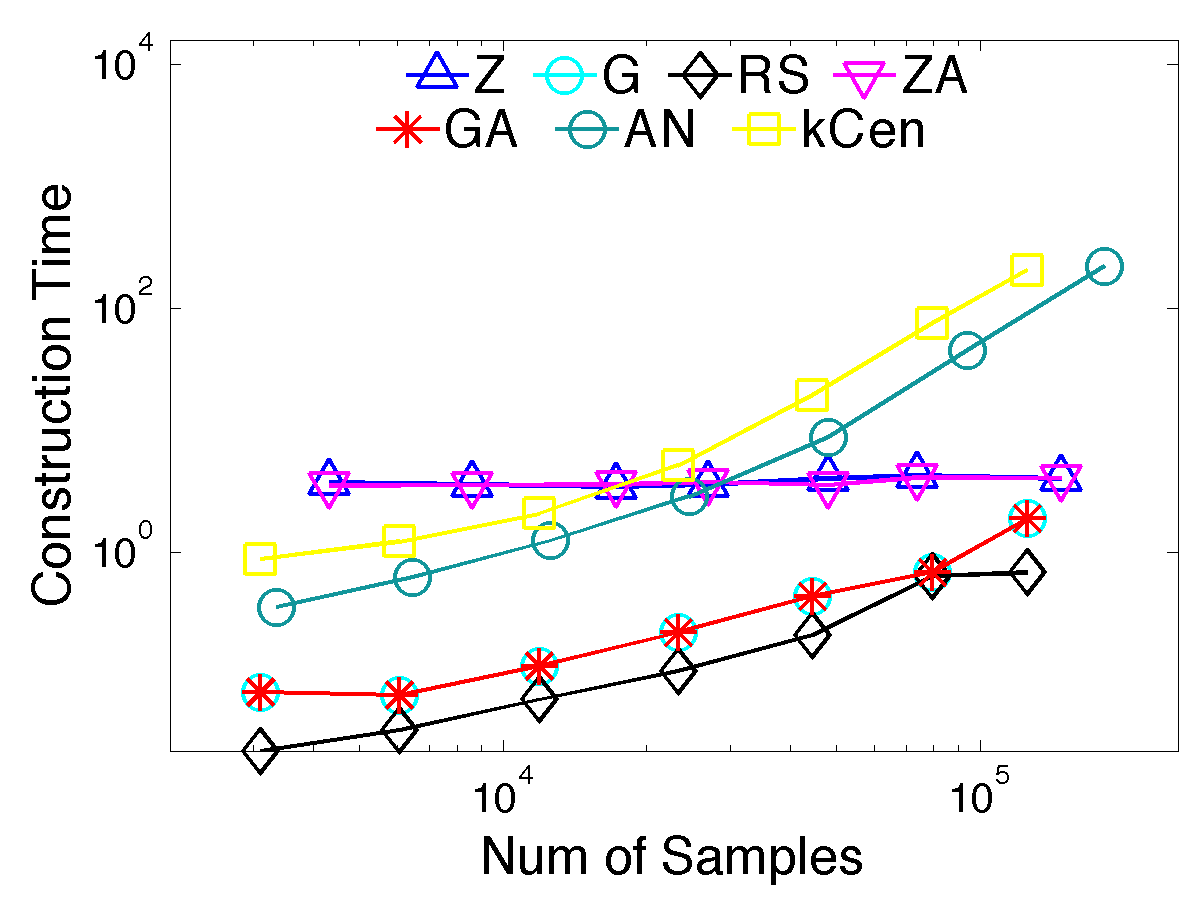}
  \includegraphics[width=0.49\linewidth]{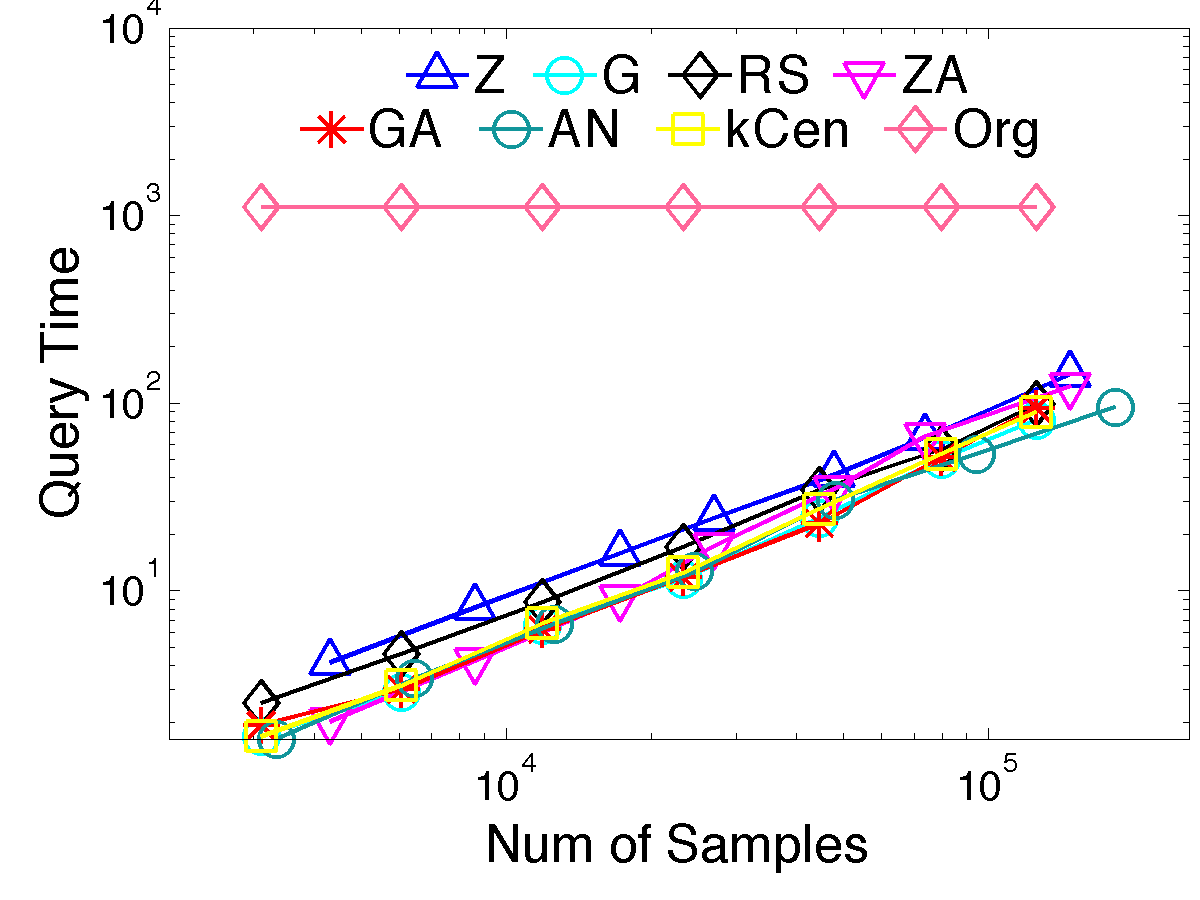}
  \vspace{-.2in}
  \caption{\sffamily Comparison of construction time and query time for real dataset with $P_x \subset \b{R}^2$.   \label{fig:iowa_time}}
\vspace{-.2in}
\end{figure}

\subsection{Efficiency of Coresets}
To show the efficiency of our methods, we compare the construction time and query time based on the coreset compared to the original dataset (denoted Org) for the dataset $P_x \subset \b{R}$. 
For both comparisons, inspired by the Improved Fast Gauss Transform~\cite{Yang2003} and other fast kernel evaluation methods, for each query point, only the neighbor points within ten bandwidth are queried to calculate the kernel regression values.  The construction time includes building the tree structure for the local data query, plus the time to generate the coreset. The query time are based on $128{,}000$ evaluation points.  

From Figure \ref{fig:housesyn_time} for $P_x \subset \b{R}$ , \textbf{Grid}, \textbf{G-Aggregate}, \textbf{Random Sample}, \textbf{Z-order} and \textbf{Z-Aggregate} have the most efficient construction times, roughly as fast as just reading the data.  Note that \textbf{k-Center} becomes quite slow for large coreset size.  Similarly,  \textbf{Grid}, \textbf{G-Aggregate}, and \textbf{Random Sample} are very efficient for $P_x \subset \b{R}^2$ (Figure \ref{fig:iowa_time}), but  \textbf{Z-order} and \textbf{Z-Aggregate} have noticeable overhead compared to the grid-based methods (and have no accuracy or analysis advantage).  
In both settings, there is also considerable time overhead to running \textbf{Aggregate-Neighbor}, which has a slight accuracy advantage for $P_x \subset \b{R}^2$ -- this it is probably only worth it if pre-processing time on these scales are not of much importance but accuracy for $P_x \subset \b{R}^2$ is.   

For the query time, all the methods improve at least $2$ orders of magnitude over using the original data.  Their query times are all about the same; this is as expected since they all produce a coreset of the same size, which can be used as proxy for the full data set in precisely the same way.  

\vspace{-.05in}
\Paragraph{Main take-away}
In conclusion, \textbf{G-Aggregate} is the best algorithm in terms of effectiveness and efficiency for $P_x \subset \b{R}$, with \textbf{Aggregate-Neighbor} has better accuracy in $P_x \subset \b{R}^2$, but has some increased overhead in construction time.  They are orders of magnitude faster than using the original data (and best among all proposed methods) and have extremely small error even for small coreset sizes (again the best among all proposed methods).  For data sets of size 1 or 2 million, they achieves very small error using only $10{,}000$ points and almost no error around $100{,}000$ points.  
They (especially \textbf{G-Aggregate}) are very simple to implement, and about as fast to construct as reading the data.  
As seen in Section \ref{sec:analysis} we prove very strong error guarantees for these methods.  

\vspace{-.1in}
 \begin{figure*}[t]
   \includegraphics[width=0.33\linewidth]{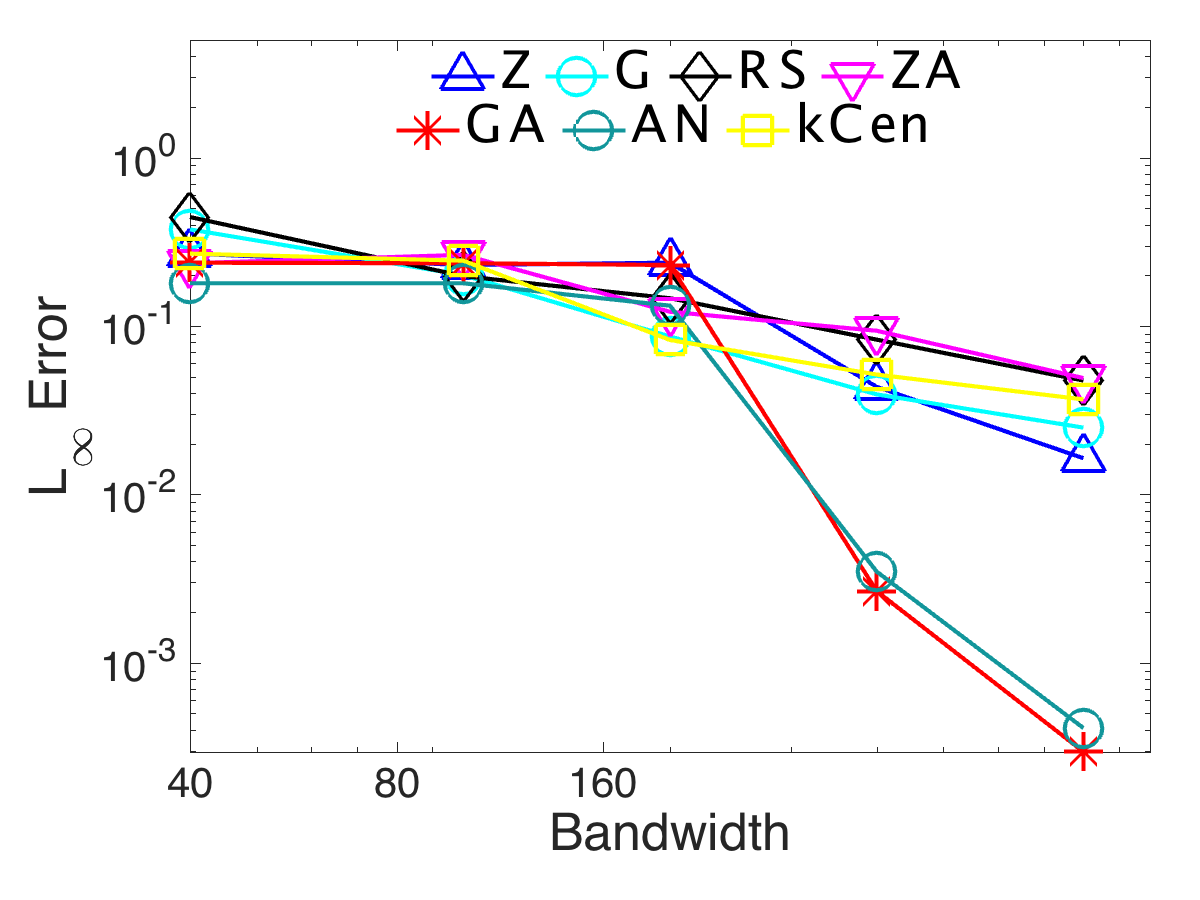}
  \includegraphics[width=0.33\linewidth]{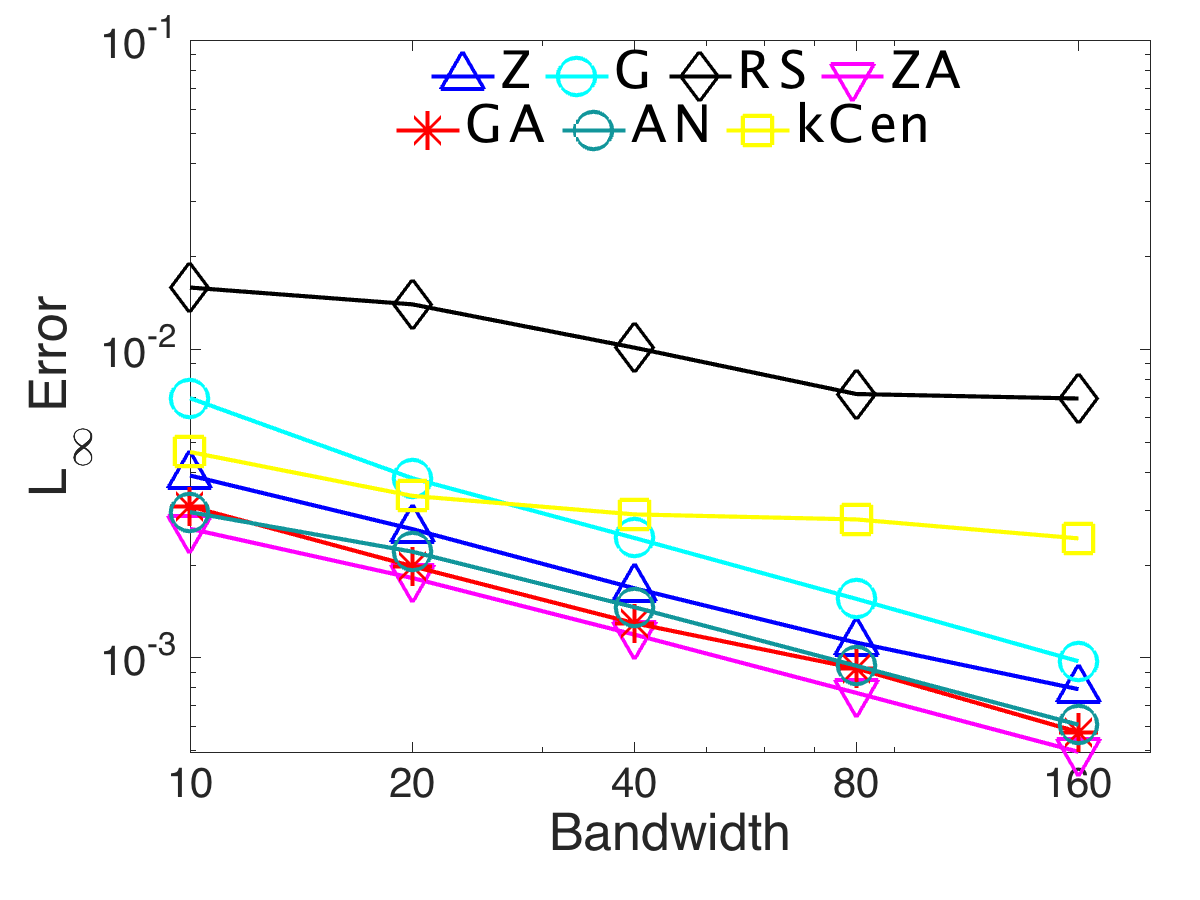}
  \includegraphics[width=0.33\linewidth]{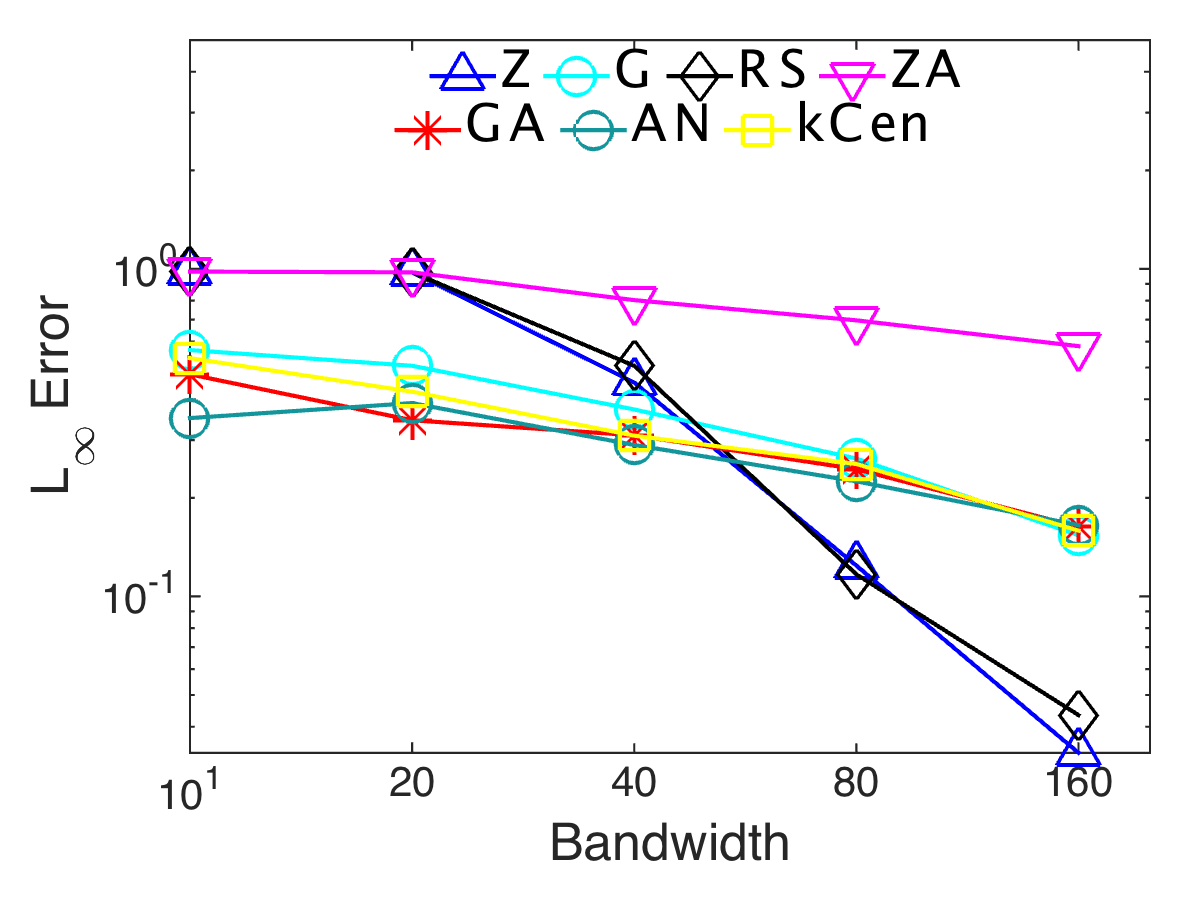}
\vspace{-.4in}
  \caption{\label{fig:syn_band}
\sffamily The relation between $L_\infty$ error and bandwidth for real data (left), synthetic data (middle) for $P_x \subset \b{R}$ and real data (right) for $P_x \subset \b{R}^2$.}
\vspace{-.25in}
\end{figure*}

\subsection{Consistency with Bandwidth}
In Figure \ref{fig:syn_band}, we test the consistency of the algorithms by varying the bandwidth.  We fix the number of evaluation point as $128{,}000$ and the median coreset size as $62{,}500$. By varying the bandwidth from $40$ to $800$ for real data (left) of $P_x \subset \b{R}$, $10$ to $160$ for synthetic data (middle) for $P_x \subset \b{R}$ and real data (right) for $P_x \subset \b{R}^2$, the errors are decreasing for all the methods.   This matches with our analysis in Section \ref{sec:analysis} and aligns with the notion that the more we smooth the data, the more stable it is, and the fewer data points we actually need.  
Again, \textbf{G-Aggregate} consistently performs the best or among the best of our methods.  The exception is the real data in $\b{R}^2$ (right), where for very large bandwidths, the simple methods \textbf{Z-order} and \textbf{Random Sample} dominate.  In this setting, the data is so smoothed these methods, which in this setting exactly or roughly amount to a random sample, work better than trying to fit gridded data to circularly smoothed estimates.   
We do not attempt to automatically choose the bandwidth, as this should be a choice of the user to determine the scale they examine the data~\cite{PZ15}.  

\vspace{-.1in}
\subsection{Progressive grid-based approaches}
We evaluate the progressive grid-based approach on the CloudLab data.  
The total coreset size is $316{,}485$, using \textbf{G-Aggregate} in each region.   And we use $256{,}000$ evaluation points.  
We evaluate the algorithm at four smoothing choices $\sigma = \{10, 15, 30, 45\}$.
For each $\sigma$, we gradually increase the window size $T$, starting at $1$ day ($86{,}400$), up to $10$ months($2.5\cdot10^7$), as shown in Figure \ref{fig:time_error_utah}.  We see that as a new region is reached, and the grid size enlarges, then so does the error.  
Also, as long as $T/\sigma$ is bounded by $4\cdot 10^4$, the error stays under $0.01$.   

\begin{figure}[]
\vspace{.1in}
\hfill
\includegraphics[width=0.8\linewidth]{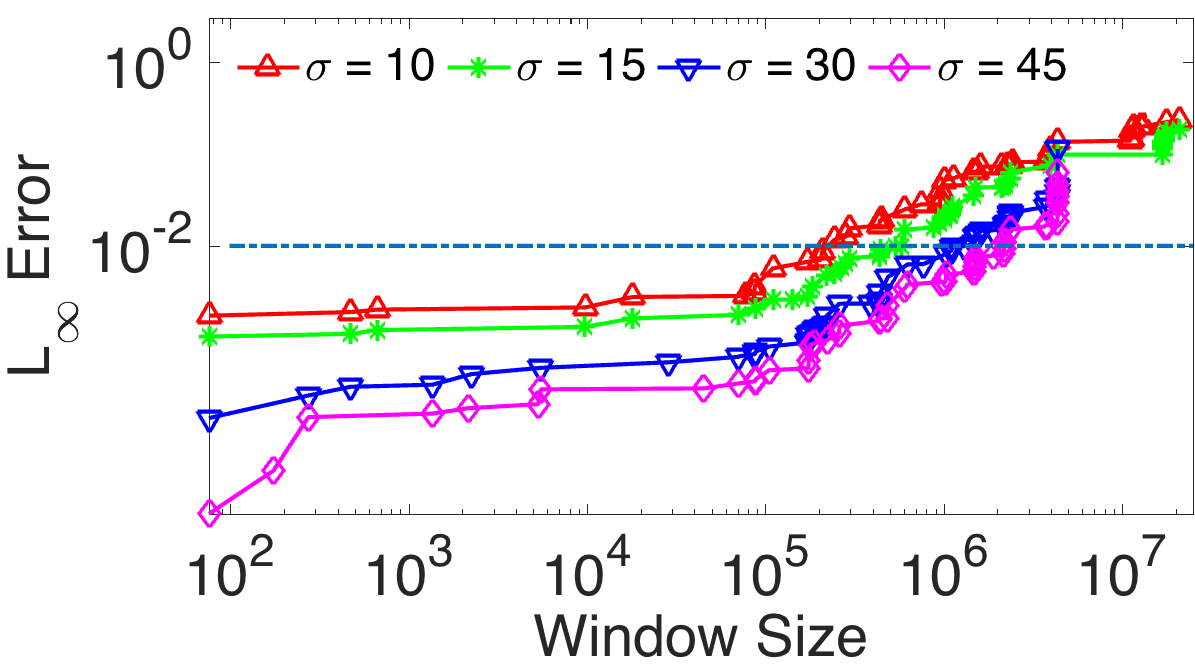} 
\hfill
\vspace{-.1in}
  \caption{\sffamily The relation between window size and $L_\infty$ error for progressive G-Aggregate method.}
   \label{fig:time_error_utah}
\end{figure}

\subsection{High dimension}
For simplicity, we only compare \textbf{G-Aggregate} and \textbf{Random Sampling}. 
When increasing the grid size, the number of empty grid increases as well, so we observe that, the size of coreset doesn't increase exponentially.  By increasing the number of grids cells from $10$ to $20$, that is a factor $2$ in each dimension, the average number of non-empty grids for CAD dataset are $\{902, 1367, 1863,2538,3150, 3791\}$ and for CASP dataset are $\{1034,1554,2122,2742,3543,4342\}$. 
The relationship of $L_\infty$ error and coreset size is shown in the left Figure \ref{fig:size_error_hd}, using bandwidths $3$ and $3.8$ respectively. The error decreases when the size of coreset increases for both methods. For the same coreset size, \textbf{Random Sampling} perform better than \textbf{G-Aggregate}, and its running time (right figure in Figure \ref{fig:size_error_hd}) is much less.  This aligns with our theoretical bounds.  For example, for grid size $20^6$, the \textbf{G-Aggregate} method takes about $250$s, while the \textbf{Random Sampling} takes only around $3$s. So we recommend the simple and fast method \textbf{Random Sampling} to generate coreset for kernel regression for high dimension datasets.

\begin{figure}[]
\vspace{-.1in}
\includegraphics[width=0.49\linewidth]{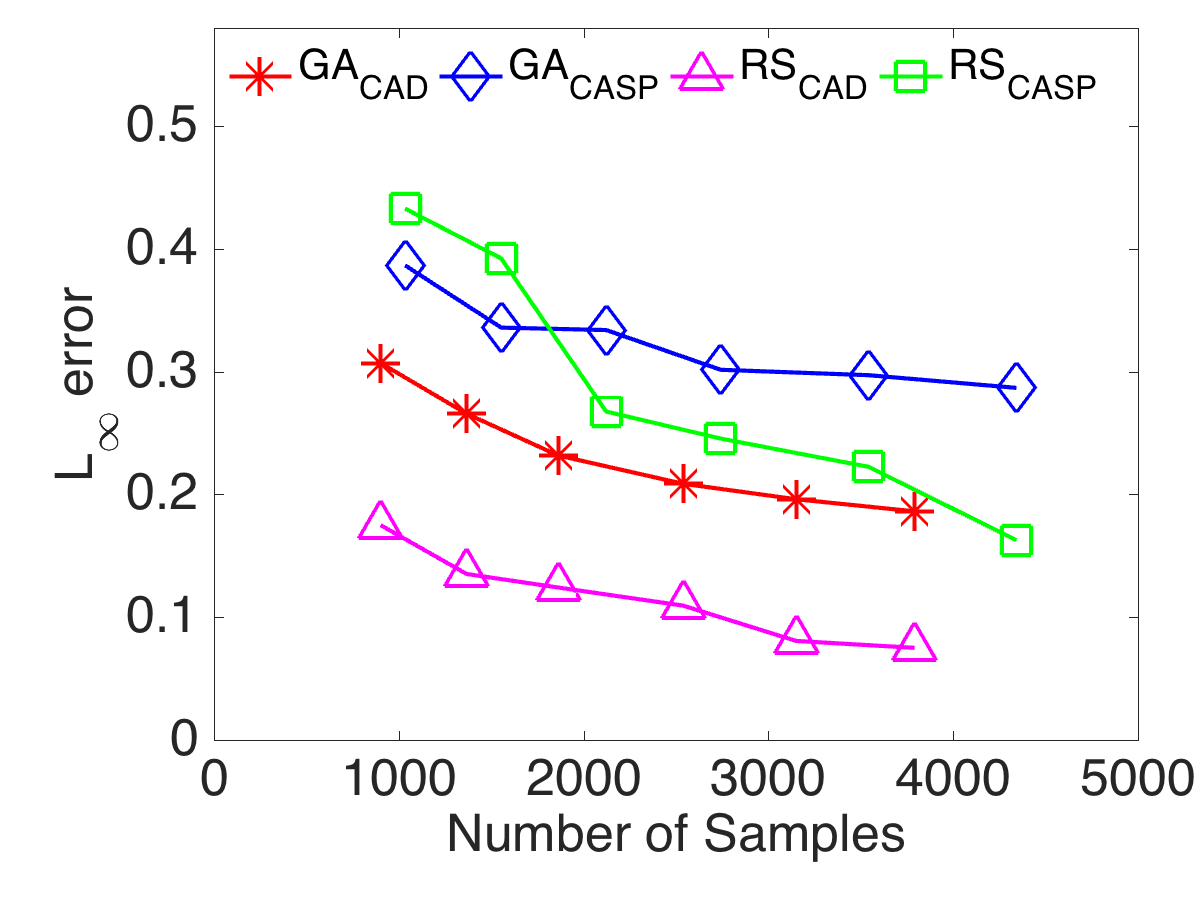} 
\includegraphics[width=0.49\linewidth]{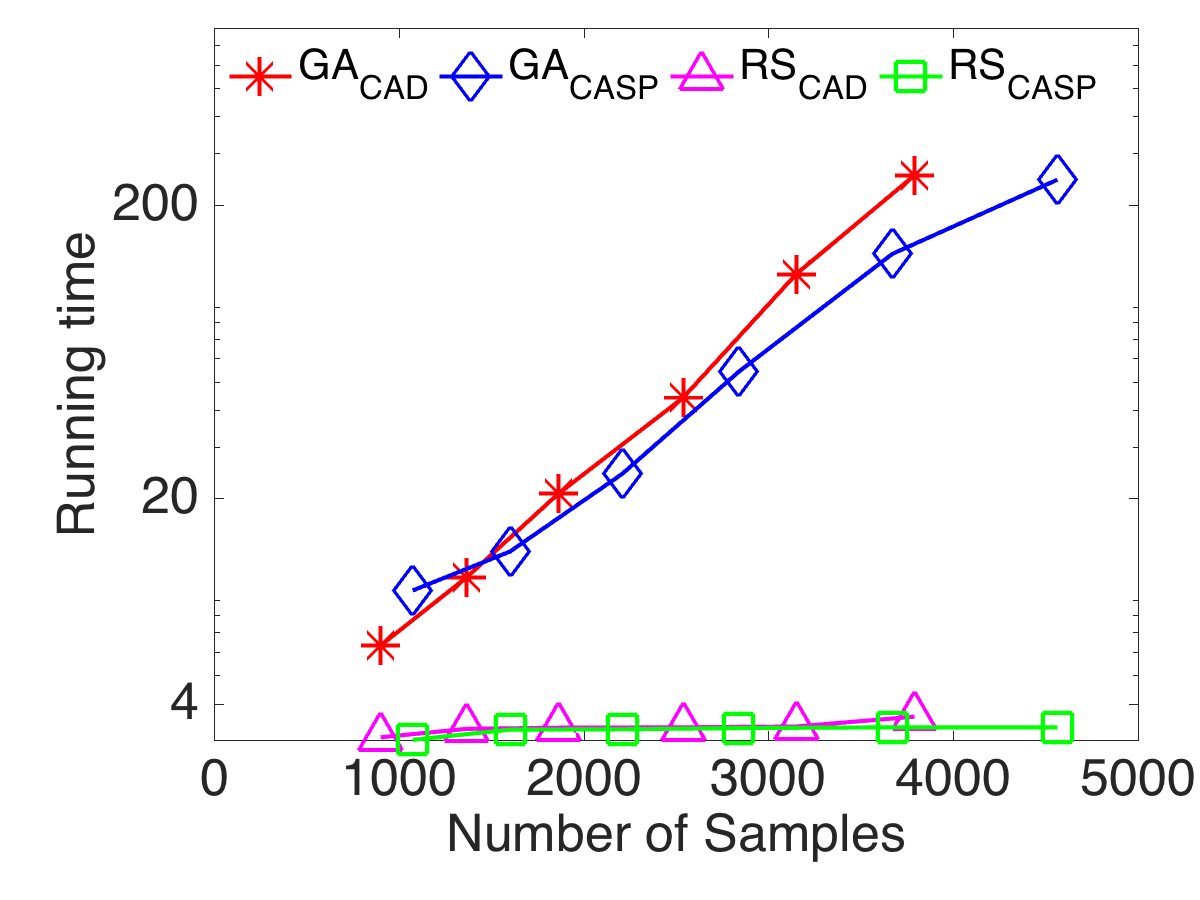} 
\vspace{-.1in}
  \caption{Left: $L_\infty$ error for coresets of high dimensional datasets. Right: Running time to generate the coresets.}
   \label{fig:size_error_hd}
\vspace{-.2in}
\end{figure}

\section{Conclusion}
We describe several algorithms for coresets for kernel regression.  Many (random sampling, order-based thinning, and grid-based thinning) are common heuristics.  As we demonstrate on data sets with millions of points, those based on grids work much better, and that small modification of aggregating and sometimes filling in sparse-neighborhood boundaries can make large difference in error reduction.  With our best methods, massive data sets can be drastically reduced in size and have negligible error.  
Find our code and data at: \url{http://www.cs.utah.edu/~jeffp/students/kernel-reg/}.

\bibliographystyle{abbrv}
\small
\bibliography{kernel-refs}
\normalsize

\newpage
\appendix


\section{Linking  and $(\rho, \eps)$-Approximations for Kernel Regression}
\label{ap:link}

\begin{theorem}
For any kernel $K : \b{R}^d \times \b{R}^d \to \b{R}^+$ linked to a range space ($\b{R}^d, \c{A}$), a $(\rho,\eps)$-approximation $S$ of $(P, \c{A})$ for $S \subset \b{R}^d$ is a $(\rho K^+,2\eps)$-approximation of $(P, K)$, where $K^+ = \max_{p,q\in P} K(p,q)$. 
\end{theorem}

\begin{proof}
We first give the definition of $\kappa$. 
For two point sets $P,S$, define a \emph{similarity} between the two point sets as \vspace{-2mm}
\[
\kappa(P,S) = \frac{1}{|P|}\frac{1}{|S|} \sum_{p \in P} \sum_{s \in S} K(p,s),
\]
and when the point set $S$ only contains one single point $s$ and a subset $P'\subset P$, we have $\kappa_P(P',s)= (1/|P|)\sum_{p\in P'}K(p,s)$.

Then we follow the same technique in proof of Theorem 5.1 in \cite{JoshiKommarajuPhillips2011}, suppose $q$ is any query point, we can sort all $p_i \in P$ in similarity to $q$ so that $p_i < p_j$ (and by notation $i < j$) if $K(p_i,q) > K(p_j,q).$ Thus any super-level set containing $p_j$ also contains $p_i$ for $i < j$. We can now consider the one-dimensional problem on this sorted order from $q$. 

We now count the deviation $D(P,S,q) =\kde_P(q) - \kde_S(q) $  from $p_1$ to $p_n$ using a charging scheme. That is each element $s_j \in S$ is charged to $g = |P| / |S|$ points in $P$. For simplicity we will assume that $g$ is an integer, otherwise we can allow fractional charges. We now construct a partition of $P$ slightly differently, for positive and negative $D(P,S,q)$ values, corresponding to undercounts and overcounts, respectively.

\Paragraph{Undercount of $\kde_S(q)$} For undercounts, we partition $P$ into $2|S|$ sets $\{P_1', P_1, P_2', P_2,...,P'_{|S|},P_{|S|}\}$ of consecutive points by the sorted order from $q$. Starting with $p_1$, we place points in set $P_j'$ or $P_j$ following their sorted order. Recursively on $j$ and $i$, starting at $j =1$ and $i = 1$, we place each $p_i$ in $P_j'$ as long as $K(p_i, q)>K(s_j, q)$ (this may be empty). Then we place the next $g$ points $p_i$ into $P_j$. After $g$ points are placed in $P_j$, we begin with $P_{j+1}'$, until all of $P$ has been placed in some set. Let $t \le |S|$ be the index of the last set $P_j$ such that $|P_j| = g$. Note that for all $p_i \in P_j$ (for $ j \le t)$ we have $K(s_j, q) \ge K(p_i, q)$, thus $\kappa_S ({\{s_j\}},q) \ge \kappa_P(P_j,q)$.  We can now bound the undercount as 
\begin{align*}
D(P,S,q) 
&= 
\sum_{j=1}^{|S|}  \big ( \kappa_P(P_j,q) \big) -  \kappa_S ({\{s_j\}},q)+\sum_{j=1}^{|S|} \kappa_P(P_j',q)\\
& \le 
\sum_{j=1}^{t+1}\kappa_P(P_j',q)
\end{align*}
since the first term is at most $0$ and $|P_j'| = 0$ for $j > t+1$. Now consider a super-level set $H \in \c{A}$ containing all points before $s_{t+1}$; $H$ is the smallest range that contains every non-empty $P_j'$. Because (for $j \le t$) each set $P_j$ can be charged to $s_j$, then $\sum_{j=1}^t |P_j \cap H| = g |S \cap H|$. And because $S$ is an $(\rho, \eps)$-approximation of $(P, \c{A})$, then 
\[
\frac{|P \cap H|}{|P|}  - \frac{|S \cap H|}{|S|} \le \eps \max\Big\{\frac{|P \cap H|}{|P|},\rho\Big\}.  
\]
Hence
\begin{align*}
\frac{1}{|P|}\sum_{j=1}^{t+1} |P_j'| & = \frac{1}{|P|}\sum_{j=1}^{t+1} |P_j' \cap H| \\
&= \frac{1}{|P|}(\sum_{j=1}^{t+1} |P_j' \cap H| + \sum_{j=1}^{t} |P_j \cap H| - g  |S \cap H| )\\
&= \frac{|P \cap H|}{|P|}  - \frac{|S \cap H|}{|S|}\le \eps \max\Big\{\frac{|P \cap H|}{|P|},\rho\Big\}. 
\end{align*}
We can now bound 
\[
D(P,S,q) \le \sum_{j=1}^{t+1}\kappa_P(P_j',q) =  \sum_{j=1}^{t+1} \sum_{p \in P_j'}\frac{K(p,q)}{|P|}.  
\]
When $\frac{|P \cap H|}{|P|} \ge \rho$, and $\rho \ge \sum_{p \in P_1'}\frac{K(p,q)}{|P|} \Big/ \eps$,
\begin{align*}
D(P,S,q)  & \le \sum_{j=1}^{t+1} \sum_{p \in P_j'}\frac{K(p,q)}{|P|} 
\le \frac{\eps}{|P|}\sum_{p\in P}K(p,q)+  \sum_{p \in P_1'}\frac{K(p,q)}{|P|}\\
& \le \eps \kde_P(q) + \eps \rho \le 2\eps \max\{\kde_P(q),\rho\}. 
\end{align*}
The second inequality is because, all the points in $P_j$ has larger $K(\cdot,q)$ values than the points in $P_{j+1}' $ and $\frac{1}{|P|}\sum_{j=1}^{t+1} |P_j'| \le \eps \frac{|P \cap H|}{|P|}$.

Finally, when $\frac{|P \cap H|}{|P|} \le \rho$,
\[
D(P,S,q)  \le \sum_{j=1}^{t+1} \sum_{p \in P_j'}\frac{K(p,q)}{|P|} 
\le \frac{1}{|P|}\sum_{j=1}^{t+1}|P_j'|K^+ \le \eps \rho K^+.  
\]

\Paragraph{Overcount of \kde$_S(q)$} The analysis for overcounts is similar to undercounts, but we partition the data in a reverse way: we partition $P$ into $2|S|$ sets $\{P_1, P_1', P_2, P_2',...,P_{|S|},P'_{|S|}\}$ of consecutive points by the sorted order from $q$ (some of the sets may be empty). Starting with $p_n$ (the furthest point from $q$) we place points in sets $P_j'$ or $P_j$ following their reverse-sorted order. Recursively on $j$ and $i$, starting at $j = |S|$ and $i = n$, we place each $p_i$ in $P_j'$ as long as $K(p_i,q) < K(s_j, q)$ (this may be empty). Then we place the next $g$ points $p_i$ into $P_j$. After $g$ points are placed in $P_j$, we begin with $P_{j-1}'$, until all of $P$ has been placed in some set. 
Let $t \le |S|$ be the index of the last set $P_j$ such that $|P_j| = g$ (the smallest such $j$). Note that for all $p_i \in P_j$ (for $j \ge t$) we have $K(s_j, q) \le K(p_i, q)$, thus $\kappa_S ({\{s_j\}},q) \le \kappa_P(P_j,q)$. We can now bound the (negative) overcount as

\begin{align*}
 D(P,S,q) &=  \sum_{j=|S|}^{t}  \big ( \kappa_P(P_j,q) \big) -  \kappa_S ({\{s_j\}},q)\\
 &+ \sum_{j=t-1}^{1}  \big ( \kappa_P(P_j,q) \big) -  \kappa_S ({\{s_j\}},q) +\sum_{j=1}^{|S|} \kappa_P(P_j',q)\\
& \ge
 \kappa_P(P_{t-1},q) -  \sum_{j={t-1}}^1 \kappa_S ({\{s_j\}},q)
\end{align*} 
since the first full term is at least $0$, as is each $\kappa_P(P_j,q)$ and $\kappa_P(P_j',q)$ term in the second and third terms. We will need the one term   $\kappa_P(P_{t-1},q)$ related to $P$.

Now using that $S$ is an $(\rho, \eps)$-sample of $(P, \c A$), we will derive a bound on $t$. We consider the maximal super-level set $H \in \c A$ such that no points $H \in P$ are in $P_j'$ for any $j$. This is the largest set where each point $p \in P$ can be charged to a point $s \in S$ such that $K(p, q) > K(s, q)$, and thus presents the smallest (negative) overcount. In this case, $H \cap P = \cap_{j=1}^w P_j$ for some $w$ and $H \cap S = \cap_{j=1}^w \{s_j\}$. Since $t \le w$, then $|H \cap P|=(w-t+1)g + |P_{t-1}| = (w-t+1)|P|/|S| + |P_{t-1}|$ and $|H \cap S| = w$. With the definition of $(\rho, \eps)$-sample, 
\[
\frac{|S \cap H|}{|S|}  - \frac{|P \cap H|}{|P|}  \le \eps \max\Big\{\frac{|P \cap H|}{|P|},\rho\Big\}.
\]
Hence
\begin{align*}
& \frac{|S \cap H|}{|S|}  - \frac{|P \cap H|}{|P|}  \\
&= 
\frac{w}{|S|}- \frac{(w-t+1)|P|/|S|}{|P|} - \frac{|P_{t-1}|}{|P|} \\
&\ge 
\frac{t-1}{|S|}- \frac{|P_{t-1}|}{|P|}.
\end{align*} 

If $\frac{|P \cap H|}{|P|} \ge \rho$ then $\frac{t-1}{|S|}- \frac{|P_{t-1}|}{|P|} \le \eps \frac{|P \cap H|}{|P|}$, which implies $\frac{|P_{t-1}|}{|P|} - \frac{t-1}{|S|}\ge -\eps \frac{|P \cap H|}{|P|}$, then 
\begin{align*}
D(P,S,q) & \ge  \kappa_P(P_{t-1},q) -  \sum_{j={t-1}}^1 \kappa_S ({\{s_j\}},q)\\
& = \frac{\kappa(P_{t-1},q)}{|P|} - \frac{ \sum_{j={t-1}}^1K(s_j,q)}{|S|}\\
& \ge -\eps \frac{\sum_{j={w}}^1\kappa(P_j,q)}{|P|} \ge -\eps \kde_P(q). 
\end{align*}
The second inequality is because $\frac{|P_{t-1}|}{|P|} - \frac{t-1}{|S|}\ge -\eps \frac{|P \cap H|}{|P|}$ and for each $s_j$ with $j \le w$, for any $p \in P_j$, $K(s_j,q) \le K(p,q)$.

If $\frac{|P \cap H|}{|P|} \le \rho$ then $\frac{t-1}{|S|}- \frac{|P_{t-1}|}{|P|} \le \eps \rho$, which implies $t-2 \le \eps \rho |S|+\frac{|S||P_{t-1}|}{|P|}-1$. Letting $p_i = \min_{i'\in P_{t-1}}K(p_i',q)$
\begin{align*}
&D(P,S,q) \\
&\ge  \kappa_P(P_{t-1},q) -  \kappa_S ({\{s_{t-1}\}},q)+ \sum_{j={t-2}}^1 \kappa_S ({\{s_j\}},q)\\
& = \frac{\kappa(P_{t-1},q)}{|P|} - \frac{ K(s_{t-1},q)}{|S|} - (\eps \rho |S|+\frac{|S||P_{t-1}|}{|P|}-1)\frac{K^+}{|S|}\\
& \ge -\eps\rho K^+ + K^+\Big(\frac{g-|P_{t-1}|}{|P|}\Big)- \frac{g\cdot K(s_{t-1},q)- \kappa(P_{t-1},q)}{|P|}\\
&\ge -\eps\rho K^+ +K^+\Big(\frac{g-|P_{t-1}|}{|P|}\Big)-K(p_i,q)\Big(\frac{g-|P_{t-1}|}{|P|}\Big) \\
&\ge -\eps\rho K^+.  
\end{align*}

So when $\frac{|P \cap H|}{|P|} \ge \rho$, $S$ is an $(\rho,2\eps)$- approximation of $(P, \c K)$, and when $\frac{|P \cap H|}{|P|} \le \rho$, it is a $(\eps \rho K^+)$-approximation of $(P, \c K)$.
\end{proof}

\end{document}